\documentclass{article}
\usepackage{fullpage}
\usepackage{graphicx}
\usepackage{url}
\usepackage{color}
\usepackage{amsthm, amsmath,amssymb}

\def\pd#1{\partial_{#1}}
\def\CC{{\bf C}}
\def\NN{{\bf N}}
\def\RR{{\bf R}}
\def\QQ{{\bf Q}}

\def\BBB{\mathcal{B}}

\def\uEE{{\hat E}}  

\def\sign{{\rm sign}\,}

\newtheorem{theorem}{Theorem}
\newtheorem{proposition}{Proposition}
\newtheorem{algorithm}{Algorithm}
\newtheorem{example}{Example}

\title{An Application of the Holonomic Gradient Method to the Neural Tangent Kernel}
\author{Akihiro Sakoda and Nobuki Takayama}
\date{2024.10.31}
\begin{document}
\maketitle

\section{Introduction}

A.Jacot et al \cite{jacot-2018} introduced a function $\Theta(x,x')$
that converges to the neural tanget kernel (NTK).
Here,  $x, x'$ are data vectors. 
In order to construct this function, we need to evalute
the expectations
\begin{eqnarray}
E_{(u,v) \sim N(0,\Lambda^{(h)})} [\sigma(u)\sigma(v)] \label{eq:expectation1} \\
E_{(u,v) \sim N(0,\Lambda^{(h)})} [{\dot \sigma}(u){\dot \sigma}(v)]. \label{eq:expectation2}
\end{eqnarray}
Here $\sigma$ is an activator function of the neural network
and 
$\Lambda^{(h)}$ is a $2 \times 2$ covariance matrix
inductively defined. See, e.g., \cite[(7)]{arora-2019}.
Each of these expectations is called a {\it dual activation} of $\sigma$ and its derivative ${\dot \sigma}$ respectively.
Note that these expectations can be expressed as definite integrals with parameters.

Attempts have been made to calculate these expectations 
for various activator functions, 
and closed forms have been found for many activator functions.
Han et al \cite{han-2022} gives several new closed forms as well as
a survey on the works on closed forms.

A system of linear partial differential equations of $n$ variables is called
a {\it holonomic system} when the dimension of its characteristic variety 
(the variety defined by the ideal generated by principal symbols) 
is $n$.
A distribution is called a {\it holonomic distribution} if it is a solution of a holonomic system.
In this paper, we note that when the activator function is a holonomic distribution,
these expectations satisfy holonomic systems of linear partial differential equations
and further show that these holonomic systems can be derived automatically by computer algebraic algorithms.
We give the following new results based on this fact.
\begin{enumerate}
\item We give a method to evaluate these expectations using a numerical method for solving linear ordinary differential equations.
This will provides a general method to calculate $\Theta$ when a new holonomic activator distribution is proposed.
\item When the activator distribution is a polynomial times a Heaviside function, this expectation can be expressed as a closed form in terms of the Gauss hypergeometric function.
\item Han et al \cite{han-2022} gives a general expression of the dual activation
for a polynomial activator function. Smooth activators have Hermite polynomial expansions.
They utilize this fact to give an approximate dual activation.
We present a computer algebra method, which is well-known among computer algebra experts, 
to derive Hermite expansions.
\end{enumerate}

The method of deriving a holonomic system and numerically evaluating 
definite integrals with parameters by its numerical analysis 
is called the holonomic gradient method (HGM) and 
has been applied to a variety of problems \cite{hgm-web}.
We refer to the book \cite[chap 6]{dojo-en} 
and papers \cite{n3ost2}, \cite{wishart-2013}
as introductory documents. 
Although methods proposed in this paper falls into the HGM,
our methods are specialized for the evaluation of 
(\ref{eq:expectation1}) and (\ref{eq:expectation2})
to make it more efficient,
which is done by 
some improvements of numerical solvers for the HGM,
by utilizing the result by Koyama and Takmura \cite{koyama-takemura-2013},
and by restriction algorithms in computer algebra
to derive holonomic systems of these expectations.

\bigbreak
Related works:
refer to \cite{han-2022} on a comprehensive survey on dual activation.

\section{Computation of $\Theta$}

Jacot et al \cite[Th 1]{jacot-2018} introduced a function
$\Theta$ that approximates the neural tangent kernel.
Arora et al \cite[Th 3]{arora-2019} gave a precise error analysis of 
the approximation.
Following these papers, we briefly summarize the procedure to construct
the function $\Theta$.

Let $f(x,\theta)$ be a neural network whose input is $x$ and parameter vector is
 $\theta$.
The neural tangent kernel (NTK) is a kernel function defined by 
\begin{equation}  \label{eq:ntk}
K(x,x')=\left\langle \frac{\partial f(x,\theta)}{\partial \theta} , \frac{\partial f(x^\prime,\theta)}{\partial \theta}\right\rangle
\end{equation}
where 
$\frac{\partial f(x,\theta)}{\partial \theta}$ is the gradient vector and
$\langle \ ,\  \rangle$ is the standard inner product.

The neural network $f$ is a composition of linear functions and activator
functions defined as follows.
Let $x \in \RR^d$ be an input
and put $g^{(0)}(x) = x$, $d_0 = d$.
Our fully connected neural network of $L$ layers is inductively defined as
follows
$$
f^{(h)}(x) = W^{(h)}\cdot g^{(h-1)} \in \RR^{d_h},\  g^{(h)} = \sqrt{\frac{c_{\sigma}}{d_h}}\sigma\left(f^{(h)}(x) \right)\in \RR^{d_h},\  h=1,2,\ldots,L
$$
Here, $W^{(h)}\in \RR^{d_{h}\times d_{h-1}}$ is a weight matrix of the $h$-th layer,
$\sigma$ is an activator function, 
$c_\sigma = \left( E_{z\sim N(0,1) \lbrack \sigma(z)^2 \rbrack}\right)^{-1}$ 
is the inverse of the expectation of $\sigma^2$ under the normal distribution
with the mean $0$ and the covariance $1$.
$\sigma((y_1, \ldots, y_{d_h})^T)$ means
$(\sigma(y_1), \ldots, \sigma(y_{d_h}))^T)$.
The output of the last layer is defined as
$$
 f(x,\theta) = f^{(L+1)}(x) = W^{(L+1)}\cdot g^{(L)}(x),\quad
 W^{(L+1)} \in \RR^{1\times d_L}.
$$
Let us introduce the function $\Theta$.
We inductively define covariance matrices $\Lambda^{(h)}(x,x')$ as follows.
\begin{equation}
 \Sigma^{(0)}(x,x^\prime) = x^T x^\prime,
\end{equation}
\begin{equation}
 \Lambda^{(h)}(x,x^\prime) = 
\begin{pmatrix}
 \Sigma^{(h-1)}(x,x) & \Sigma^{(h-1)}(x,x^\prime) \\
 \Sigma^{(h-1)}(x^\prime,x) & \Sigma^{(h-1)}(x^\prime, x^\prime)\\
\end{pmatrix}
\end{equation}
\begin{equation}\label{sigma}
 \Sigma^{(h)}(x,x^\prime) = c_\sigma E_{(u,v)\sim N(0,\Lambda^{(h)})}\lbrack \sigma(u)\sigma(v)\rbrack
\end{equation}
\begin{equation}\label{sigma_dot}
 \dot\Sigma^{(h)}(x,x^\prime) = c_\sigma E_{(u,v)\sim N(0,\Lambda^{(h)})}\lbrack \dot\sigma(u) \dot\sigma(v)\rbrack
\end{equation}
Here, $\dot{\sigma}$ is the derivative of the activator function $\sigma$.
The function $\Theta(x,x^\prime)$ approximating the neural tangent kernel
is defined as
\begin{equation}  \label{eq:theta}
 \Theta(x,x^\prime) = \Theta^{(L)}(x,x^\prime)=\sum_{h=1}^{L+1}\left( \Sigma^{(h-1)}(x,x^\prime) \prod_{h^\prime = h}^{L+1} \dot\Sigma(x,x^\prime)\right)
\end{equation}
Here, we put $\dot\Sigma^{(L+1)}(x,x^\prime) = 1$.

Assume that all elements of parameter $\theta$ are independent and identically
distributed as $N(0,1)$.
When the width of the neural network is infinite $d_1,d_2, \ldots , d_L \rightarrow \infty$,
the following theorems hold.

\begin{theorem} \label{th:arora-convergence} {\rm \cite[Th 3.1]{arora-2019}}
 Fix $\epsilon > 0$ and $\delta \in (0,1)$. Suppose $\sigma(z) = \max(0,z)$ and $\min_{h\in{[L]}}d_{h} \geq \Omega(\frac{L^6}{\epsilon^4} \log(L/\delta))$. Then for any inputs $x,x^\prime \in \RR^{d_0}$ such that $\|x\| \leq 1,\|x^\prime\| \leq 1$, with probability at least $1-\delta$ we have
\begin{equation}
 \left | \left\langle \frac{\partial f(x,\theta)}{\partial \theta} , \frac{\partial f(x^\prime,\theta)}{\partial \theta}\right\rangle - \Theta^{(L)}(x,x^\prime) \right| \leq (L+1)\epsilon
\end{equation}
\end{theorem}

\begin{theorem} \label{arora-equivalence-nn-ntk} {\rm \cite[Th 3.2]{arora-2019}}
 Suppose $\sigma(z) = \max(0,z)$, $1/\kappa = poly(1/\epsilon,\log(n/\delta))$ and $d_1 = d_2 = \cdots = d_L = m$ with $m \geq poly(1/\kappa, L , 1/\lambda_0, n, log(1/\delta))$. Then for any $x_{te} \in \RR^d$ with $\|x_{te}\| = 1$ with probability at least $1-\delta$ over the random initialization, we have
\begin{equation}
 |f_{nn}(x_{te})-f_{ntk}(x_{te})| \leq \epsilon
\end{equation}
\end{theorem}
These theorems are error analysis for the ReLU activator ${\rm max}(0,z)$.
As to convergence theorems for other activators,
see, e.g., \cite[Th 1]{jacot-2018}, \cite{yang-2019}.

Our definition of the neural network is a composite of linear maps 
(affine maps without bias terms) and activator functions. 
Note that when there are bias terms \cite{jacot-2018}, we may set 
\begin{equation}
 \Sigma^{(0)} = x^T x^\prime + \beta^2
\end{equation}
and
\begin{equation}
 \Sigma^{(h)}(x,x^\prime) = c_\sigma E_{(u,v)\sim N(0,\Lambda^(h))}\lbrack \sigma(u)\sigma(v)\rbrack + \beta^2.
\end{equation}
Here, $\beta$ is a hyperparameter.

\section{Holonomic activator distribution, HGM, and HIE}

Let $\sigma(u)$ be an activator distribution.
When it satisfies a linear ordinary differential equation (linear ODE) with polynomial coefficients,
it is called {\it a holonomic activator distribution} or 
{\it a holonomic activator function}\/.
Any holonomic activator distribution has finite number of poles as a function
on the complex plane, because the pole locus is the zero of
the leading coefficient of the ODE.
The derivative of a holonomic activator distribution also satisfies 
a linear ODE with polynomial coefficients.

The ReLU activator distribution $\sigma(u)$ satisfies
$ (u \pd{u}-1) \bullet \sigma(u)=0$
($\sigma(u)$ is annihilated by $u \pd{u}-1$)
where $\pd{u} = \frac{d}{du}$ and $\bullet$ means the action of a differential
operator to a distribution.
The derivative ${\dot \sigma}(u)$ is annihilated by $u \pd{u}$.

Here is a list of holonomic activator distributions from the list
of Wikipedia article of activator functions:
binary step, rectified linear unit (ReLU), Gaussian error linear unit (GeLU), exponential linear unit (ELU), 
scaled exponential linear unit (SELU),
Leaky rectified linear unit (Leaky ReLU),
parametric rectified linear unit (PReLU),
Gaussian.
Note that the sigmoid function $\frac{1}{1+e^{-x}}$ is {\it not}
a holonomic activator distribution.
Because it has infinitely many poles at $x=\sqrt{-1}( \pi + 2n \pi)$, 
$n \in {\bf Z}$ in the complex plane.

We consider the expectation 
$E_{(u,v) \sim N(0,\Sigma)} [\sigma(u) \sigma(v)]$
where $N(0,\Sigma)$ is the $2$ dimensional normal distribution
of the average $0$ and the covariance $\Sigma$.
Put $x=-\frac{1}{2} \Sigma^{-1}$ where $x$ is the matrix
$\left(\begin{array}{cc}
x_{11} & x_{12} \\
x_{21} & x_{22} \\
\end{array}\right)
$, $x_{12}=x_{21}$.
The expectation is written as $g(x)/Z(x)$
where
\begin{equation} \label{eq:expectation-sigma2}
  g(x)=
  \int_{\RR^2} \sigma(u) \sigma(v)
  \exp(x_{11} u^2 + 2 x_{12} uv + x_{22} v^2) dudv
\end{equation}
and
\begin{equation}  \label{eq:normalizing_const}
  Z(x)=\int_{\RR^2} \exp(x_{11} u^2 + 2 x_{12} uv + x_{22} v^2) dudv 
   = \frac{\pi}{\sqrt{x_{11}x_{22}-x_{12}^2}}. 
\end{equation}
We will call $g(x)$ the unnormalized expectation
and
we denote the unnormalized expectation 
by $\uEE$ as
\begin{equation}  \label{eq:unnormalized-expectation}  
 \uEE[\sigma_1(u)\sigma_2(v)]
=  \int_{\RR^2} \sigma_1(u) \sigma_2(v)
  \exp(x_{11} u^2 + 2 x_{12} uv + x_{22} v^2) dudv
\end{equation}
for random variables $\sigma_1(u)$ and $\sigma_2(v)$.
In this paper, we evaluate this $\uEE$ as the function of $x$
unlike other literature.
The relationship with the expectation value 
expressed by $\Sigma$ is 
\begin{equation} \label{eq:relation_sigma_and_x}
E_{(u,v) \sim N(0,\Sigma)}[\sigma_1(u) \sigma_2(v)]=\uEE[\sigma_1(u)\sigma_2(v)]
\frac{\sqrt{{\rm det}(x)}}{\pi}, \quad \Sigma=-\frac{1}{2}x^{-1}.
\end{equation}
This expectation $E_{(u,v) \sim N(0,\Sigma)}[\sigma_1(u) \sigma_2(v)]$
is often denoted by 
\begin{equation} \label{eq:k_sigma}
k_{\sigma_1 \sigma_2}(c_1,c_2,r), \quad 
\Sigma=\left(
\begin{array}{cc}
 c_1^2 & c_1 c_2 r \\
 c_1 c_2 r & c_2^2
\end{array}
\right), c_i > 0
\end{equation}
to express the dual activation in other literature.
When $\sigma_1=\sigma_2$, $k_{\sigma_1 \sigma_2}$ is denoted by $k_\sigma$.
See, e.g., \cite{han-2022}.

Let 
$D_n = {\bf C}\langle x_1, \ldots, x_n, \pd{1}, \ldots, \pd{n} \rangle$ 
be the ring of differential operators
where $\pd{i} = \frac{\partial}{\partial x_i}$.
Let $\ell = \sum_{(\alpha,\beta) \in E} c_{\alpha\beta} x^\alpha \pd{}^\beta$
be an element of $D_n$
where $c_{\alpha\beta} \in {\bf C}$,
$x^\alpha = \prod_{i=1}^n x_i^{\alpha_i}$,
$\pd{}^\beta = \prod_{i=1}^n \pd{i}^{\beta_i}$,
and
$E$ is a finite subset of ${\bf Z}_{\geq 0}^{2n}$.
A left ideal $I$ in $D_n$ is called {\it a holonomic ideal}
or {\it a holonomic system} (of linear PDE's)
when the dimension of the zero set of the ideal generated by
the principal symbols of $I$ is $n$.
For example, the principal symbol of $x_1 \pd{1}^2 + 1$ is
$x_1 \xi_1^2 \in \CC[x_1,\xi_1]$ and ${\rm dim}\, V(x_1 \xi_1^2)=1$.
Then the left ideal generated by $x_1 \pd{1}^2+1$ in $D_1$ is a holonomic ideal.
See, e.g., \cite[6.4, 6.8]{dojo-en} and \cite{SST} on the notion of
a holonomic ideal.
A function (or a distribution) is called a {\it holonomic function} (or a holonomic distribution)
when it is annihilated by a holonomic ideal.
The following theorem by I.N.Bernstein \cite{bernstein-1972} is the theoretical foundation
of our method.
\begin{theorem} \cite{bernstein-1972}, see also, e.g., \cite[Th 6.10.8]{dojo-en}. \\
If the left ideal $I$ of $D_n$ is holonomic, then 
the intersection of the sum of
left ideal and right ideal and $D_{n-1}$ 
\begin{equation}  \label{eq:integration_ideal}
(I + \pd{n} D_n) \cap D_n
\end{equation}  
is a holonomic ideal in $D_{n-1}$.
\end{theorem}
Roughly speaking, the theorem implies that
if $f$ is a holonomic function in $n$ variables,
then $\int_\RR f dx_n$ is a holonomic function in $n-1$ variables.
An algorithm of construct the {\it integration ideal} (\ref{eq:integration_ideal})
is given by T.Oaku \cite{oaku-1997} (see also, e.g., \cite[Chap 6]{dojo-en}).

Let $R_n$ be the rational Weyl algebra (the ring of differential operators
with rational function coefficients
${\bf C}(x)\langle \pd{1}, \ldots, \pd{n} \rangle$, ${\bf C}(x)={\bf C}(x_1, \ldots, x_n)$).
It is known that when $I$ is holonomic, 
then $r:={\rm dim}_{{\bf C}(x)} R_n/(R_n I)$ is finite.
The dimension $r$ is called the {\it holonomic rank} of $I$.
The holonomic rank is equal to the dimension of the analytic solutions of $I$
at a generic point.
Let $s_1=1, s_2, \ldots, s_r$ be a basis of $R_n/(R_n I)$ regarded 
as a vector space over ${\bf C}(x)$.
When they are monomials of $\pd{}$, they are called 
{\it standard monomials}\/.
Then, $\pd{i} s_j$ can be expressed as a linear combination of $s_k$'s 
as $\pd{i} s_j = \sum_{k=1}^r p^i_{jk}(x) s_k$
in $R_n/(R_n I)$.
The rational functions $p^i_{jk}$ can be obtained by a Gr\"obner basis
computation (see, e.g., \cite[6.1, 6.2]{dojo-en}).
If a function $f$ is annihilated by the left ideal $I$,
then $F=(f, s_2 \bullet f, \ldots, s_r \bullet f)^T$
satisfies
\begin{equation} \label{eq:pfaffian}
\frac{\partial F}{\partial x_i} = P_i F
\end{equation}
where $P_i$ is a $r \times r$ matrix $P_i=(p^i_{jk})$.
The equation is called {\it a Pfaffian system}.
It is also expressed as
\begin{equation} \label{eq:pfaffian-differential-form}
  dF = (P_1 dx_1 + \cdots + P_n dx_n) F.
\end{equation}
It is well-known that
an ODE of the rank $r$ and the independent variable $z$ 
can be translated to a system of first order ODE
$\pd{z} \bullet F= P(z) F$ where $P(z)$ is $r \times r$ matrix.
A Pfaffian system associated to a holonomic system is a generalization
of this system.
See, e.g., \cite[\S 6.2]{dojo-en}.

A holonomic gradient method (HGM) utilizing the construction above
was introduced in \cite{n3ost2} and \cite{wishart-2013}.
It gives an algorithmic method to evaluate normalizing constant and 
expectations.
The HGM is performed by the following 3 steps.
\begin{algorithm} \label{alg:hgm} 
{\rm HGM (\cite{n3ost2}, \cite{wishart-2013}, \cite[6.5, 6.11]{dojo-en}).} 
\begin{enumerate}
\item Derive a holonomic ideal and a Pfaffian system satisfied by a definite integral $e(x)$ 
with parameter $x$, e.g., $e(x)=\uEE[\sigma_1(u)\sigma_2(v)]$.
\item Evaluate $e(x)$ and its derivatives 
at a special point $x=x_0$.
\item Solve numerically the Pfaffian system with values obtained in the step 2.
\end{enumerate}
\end{algorithm}
A difference analogy of the above algorithm is called
{\it difference HGM}\/, which will be discussed in Section \ref{sec:Hermite_expansion}.

Our algorithm to evaluate $\uEE[\sigma_1(u)\sigma_2(v)]$ 
follows the general scheme of the HGM, but is more specialized for 
computing the expectation of holonomic activator distributions.
The specialization is based on the following fact by Koyama-Takemura.

\begin{theorem} {\rm \cite[Th 1, 2]{koyama-takemura-2013}} \label{koyama-takemura-th}\\
If a tempered distribution $f(t)$ on $\RR^d$ is annihilated by
$P_1, \ldots, P_s$,
then the integral
\begin{equation}
\int_{\RR^d} f(t) \exp\left( \sum_{i,j=1}^d t_i x_{ij} t_j + \sum_{i=1}^d t_i y_i \right) dt_1 \cdots dt_d
\end{equation}
is annihiated by
\begin{eqnarray}
\varphi(P_k),&& \quad 1 \leq k \leq s, \label{eq:19} \\
\pd{x_{ij}}-2\pd{y_i}\pd{y_j},&& \quad 1 \leq i < j \leq d, \label{eq:20}\\
\pd{x_{ii}}-\pd{y_i}^2,&& 1 \leq i \leq d. \label{eq:21}
\end{eqnarray}
Here, $x_{ij}=x_{ji}$ and $\varphi(t_i)=\pd{y_i}$ and
$\varphi(\pd{t_i})=-y_i-2\sum_{k=1}^d x_{ik}\pd{y_k}$.
If the operators $P_1, \ldots, P_s$ generate a holonomic ideal,
then (\ref{eq:19}), (\ref{eq:20}), (\ref{eq:21}) generate a holonomic ideal. 
\end{theorem}

\begin{algorithm} \label{alg:1} \ \\
Input: Linear ODE $\ell_1$ and $\ell_2$ annihilating $\sigma_1(u)$
and $\sigma_2(u)$ respectively. A curve on the $x$ space. \\
Output: Values of $\uEE[\sigma_1(u)\sigma_2(v)]$ (\ref{eq:unnormalized-expectation}) on a curve.
\begin{enumerate}
\item Apply \cite[Th 2]{koyama-takemura-2013} (Theorem \ref{koyama-takemura-th}) to the left ideal
generated by $\ell_1$ and $\ell_2$ in ${\bf C}\langle u,v, \pd{u}, \pd{v}\rangle$ and obtain a holonomic ideal $I_1$ in ${\bf C}\langle x_{11},x_{12},x_{22},y_1,y_2,\pd{11},\pd{12},\pd{22},\pd{1},\pd{y}\rangle$.
\item Apply a restriction algorithm \cite{oaku-1997} to find generators of
$I_2 :=I_1 \cap {\bf C}\langle x_{11},x_{12},x_{22},\pd{11},\pd{12},\pd{22}\rangle$.
\item Translate $I_2$ to a Pfaffian system.
\item Evaluate initial values of $F$ at $x_{11}=-1,x_{12}=0,x_{22}=-1$
or around this point by the series of Proposition \ref{prop:series_at_ip}.
\item Solve the Pfaffian system numerically on a given curve.
\end{enumerate}
\end{algorithm}

Although the restriction algorithm of the step 2 works for any
holonomic input on computer algebra systems in principle,
it would be better if the ideal $I_2$ could be determined by a calculation by hand.
In fact, following the steps 1 and 2 of Algorithm \ref{alg:1} by hand,
we have the following theorem, which expresses the dual activation
in terms of the Gauss hypergeometric function
${}_2F_1(\alpha,\beta,\gamma;z)$.

\begin{theorem}  \label{th:holonomic-umvn}
Let $m,n$ are non-negative integers and $Y(u)$ the Heaviside function.
\begin{enumerate}
\item The integrals $\uEE[u^m v^n](x_{11},x_{12},x_{22})$ 
and $\uEE[u^m v^n Y(u)Y(v)](x_{11},x_{12},x_{22})$ 
satisfy the GKZ hypergeometric system (see, e.g., \cite{SST})
\begin{eqnarray} \label{eq:holonomic-umvn}
&&2x_{11}\pd{11} + x_{12}\pd{12} + m + 1, \\
&&x_{12}\pd{12} + 2 x_{22}\pd{22} + n +1, \\
&& 4 \pd{11}\pd{22}-\pd{12}^2
\end{eqnarray}
\item 
Assume $x_{11}, x_{22}<0$ 
and $ 0 \leq \frac{x_{12}^2}{x_{11} x_{22}}< 1$.
Then, the solution space of the GKZ system above is spanned by
\begin{eqnarray} 
\varphi_1&:=&
(-x_{11})^{-\alpha} (-x_{22})^{-\beta}
{}_2F_1\left( \alpha, \beta, \frac{1}{2}; z \right)  \label{eq:umvn-by-2F1} \\
\varphi_2&:=&
(-x_{11})^{-\alpha} (-x_{22})^{-\beta} \sqrt{z}\, \sign(x_{12})\,
{}_2F_1\left(\alpha+\frac{1}{2},\beta+\frac{1}{2},\frac{3}{2};z\right) \label{eq:umvn-by-2F1-2}
\end{eqnarray}
where 
\begin{equation}
\alpha = \frac{1+m}{2}, \beta=\frac{1+n}{2},
z=\frac{x_{12}^2}{x_{11} x_{22}}
\end{equation}
and
$\sign(x)$ is the sign of $x$.
\item 
Assume $x_{11}, x_{22}<0$ 
and $\frac{x_{12}^2}{x_{11} x_{22}}< 1$.
When $m, n$ are even numbers,
the integral $\uEE[u^m v^n](x_{11},x_{12},x_{22})$ is equal to
$\Gamma(\alpha)\Gamma(\beta) \varphi_1$.
If both $m, n$ are odd numbers,
the integral $\uEE[u^m v^n](x_{11},x_{12},x_{22})$ is equal to
$\frac{1}{2}mn \Gamma\left(\alpha-\frac{1}{2}\right) 
               \Gamma\left(\beta-\frac{1}{2}\right) \varphi_2$. 
If either $m$ or $n$ is odd,
then the integral is equal to $0$.
\item 
Assume $x_{11}, x_{22}<0$ 
and $\frac{x_{12}^2}{x_{11} x_{22}}< 1$.
The integral $\uEE[u^m v^nY(u)Y(v)](x_{11},x_{12},x_{22})$ 
is equal to
\begin{equation} \label{eq:rectified_umvn}
 \frac{1}{4}\Gamma(\alpha) \Gamma(\beta)\varphi_1
+\frac{1}{2}\Gamma\left(\alpha+\frac{1}{2}\right)\Gamma\left(\beta+\frac{1}{2}\right) \varphi_2
\end{equation}
\end{enumerate}
\end{theorem}

A proof of this theorem is given in Appendix \ref{sec:proof-th:holonomic-umvn}.
Note that we have
\begin{equation}  \label{eq:hg-hhh}
{}_2F_1((1+m)/2,1/2,1/2;z)=(1-z)^{-1/2-m/2},
\end{equation}
\begin{equation}  \label{eq:hg-1-1-h}
{}_2F_1(1,1,1/2;z)=\left( 1+\frac{\sqrt{z}\arcsin(\sqrt{z})}{\sqrt{1-z}}\right)(1-z)^{-1},
\end{equation}
\begin{equation}  \label{eq:hg-32-32-32}
{}_2F_1(3/2,3/2,3/2;z)=(1-z)^{-3/2}
\end{equation}
and
\begin{eqnarray}
 \frac{1}{a}(z \pd{z}+a) \bullet {}_2F_1(a,b,1/2;z)
&=& {}_2F_1(a+1,b,1/2;z) \label{eq:contiguity-a} \\
\frac{1}{b}(z \pd{z}+b) \bullet {}_2F_1(a,b,1/2;z)
&=& {}_2F_1(a,b+1,1/2;z) \label{eq:contiguity-b},
\end{eqnarray}
which are called contiguity relations.
These identities give a closed form of (\ref{eq:umvn-by-2F1})
when $m, n$ are given.
A closed form of the dual activation of a polynomial activator
is given by Han et al \cite[Theorem 1]{han-2022}.
Note that
$(\sum_{i=0}^q a_i u^i)(\sum_{j=0}^q a_j v^j)
= \sum_{i,j=0}^q a_i a_j u^i v^j$.
Then, our theorem gives a different closed form expression of
the dual activation for a polynomial activator
$\sum_{i=0}^q a_i u^i$.
Analogously, our formula gives the dual activation for a rectified polynomial activation
$ \sigma(u) = \left(\sum_{i=0}^q a_i u^i\right) Y(u)$, because we have
\begin{equation} \label{eq:rectified_poly_prod}
\sigma(u) \sigma(v) = \sum_{i,j=0}^q a_i a_j u^i v^j Y(u) Y(v) .
\end{equation}
The dual activation for a monomial is given in \cite[F.7]{han-2022}
by the hypergeometric function ${}_2F_1$.
This formula is a special case of our theorem in a different form.
A closed from for a rectified monomial is known \cite{cho-2009}.
Our formula for the rectified polynomial generalizes it and
seems to be new as long as we know.

Let us come back to the general algorithm of the HGM.
We use the following proposition to perform the step 4.
\begin{proposition}  \label{prop:series_at_ip}
Series expansion of $\uEE [\sigma_1(u) \sigma_2(v)]$
at $(x_{11},x_{12},x_{22})=(-1,0,-1)$ is
$\sum_{k \in \NN_0^3} c_k x^k$, $x^k = x_{11}^{k_{11}} x_{12}^{k_{12}} x_{22}^{k_{22}}$
where
\begin{eqnarray}
c_k = \frac{2^{k_{12}}}{k_{11}! k_{12}! k_{22}!}
&\times&\int_{-\infty}^\infty u^{2k_{11}+k_{12}} \sigma_1(u) \exp(-u^2)du \nonumber \\
&\times&\int_{-\infty}^\infty v^{2k_{22}+k_{12}} \sigma_2(v) \exp(-v^2)dv. \label{eq:series_at_ip}
\end{eqnarray}
\end{proposition}

The holonomic system can also be used to obtain an approximate expression of
the expectation 
$\uEE[\sigma_1(u)\sigma_2(v)]$
in terms of a set of basis functions
by the sparse interpolation and extrapolation method B \cite{num-hgm-2021}.
We call the following algorithm the {\it holonomic interpolation/extrapolation} method (HIE).
\begin{algorithm} \label{alg:sib} \ \\
Input: Linear ODE $\ell_1$ and $\ell_2$ annihilating $\sigma_1(u)$
and $\sigma_2(u)$ respectively. A set of functions $B=\{ e_\beta(x) \,|\, \beta \in \BBB \}$ 
on the $x$ space. 
A numerical integration scheme $(t_j,T_j)$ (evaluation points $\{t_j\}$ and positive weights $\{T_j\}$). 
$\gamma$-th derivative values $\{q_k^{(\gamma)}\}$ of the expectation $\uEE[\sigma_1(u)\sigma_2(v)]$
at $\{p_k\,|\, k=1, \ldots, r\}$ in the $x$ space.  \\
Output: An approximation of $\uEE[\sigma_1(u)\sigma_2(v)]$ (\ref{eq:unnormalized-expectation}) in terms of the set of functions $B$.
\begin{enumerate}
\item Apply the same procedure of the first two steps of Algorithm \ref{alg:1}.
\item Let $\ell_i$, $i=1, \ldots, s$ be generators of the left ideal $I_2$.
\item Put $f(x)=\sum_{\beta \in {\BBB}} f_\beta e_\beta(x)$ where
$f_\beta$ are unknown coefficients.
\item Minimize 
\begin{equation} \label{eq:numerical_integral}
\ell(\{f_\beta \}):=\sum_{i=1}^s \sum_{j} T_j \left| \sum_{\beta \in {\BBB}} f_\beta(\ell_i e_\beta)(t_j)\right|^2
\end{equation}  
under the constraints at data points
\begin{equation} \label{eq:constraints}
 \sum_\beta f_\beta e_\beta^{(\gamma)}(p_k)=q_k^{(\gamma)}, \quad k=1, \ldots, r, \alpha \in \Gamma
\end{equation}
where $e_\beta^{(\alpha)}$ is $\pd{\alpha}\bullet e_\beta$
($\gamma$-th derivative of $e_\beta$).
\item Return $\sum_{\beta} f_\beta e_\beta(x)$.
\end{enumerate}
\end{algorithm}
We give remarks about this algorithm.

Since the constraints are linear, we can parametrize the space of $f_\beta$'s 
by an affine map and reduce the problem to a least square problem
with no constraint.
Or, the loss function for the minimization may be set as
\begin{equation} \label{eq:numerical_integral_and_constraints}
\ell(\{f_\beta \}):=\sum_{i=1}^s \sum_{j} T_j \left| \sum_{\beta \in {\BBB}} f_\beta(\ell_i e_\beta)(t_j)\right|^2 + \mu \sum_{k=1}^r  \left| \sum_\beta f_\beta e_\beta(p_k)-q_k \right|^2
\end{equation}  
to transform the minimization problem to that with no constraint.
It is also a least square problem.
Here, $\mu$ is a paramter. The larger the paramter $\mu$, the closer the solution is to the given values of $f$ at $\{ p_k \}$.

There are various ways to select a set of basis functions.
For example, if we can find a set of fundamental solutions
of $I_2 \bullet e_\beta=0$,
the problem is reduced to find a best set of coefficients
$f_\beta$ satisfying the constraints (\ref{eq:constraints})
approximately.
In particular, if the basis functions are a basis of series solutions 
at a point $p_k$ and given values are derivatives standing 
for standard monomials at $x=p_k$,
then our algorithm constructs the series expansion of the expectation
at $x=p_k$.

Finally, we briefly note some other applications of our holonomic approach.
Holonomic systems for $\uEE$ can be utilized to derive several formulas
of the function $\uEE$ other than obtaining series expressions.
For example, it can be used in the following ways;
finding a higher order ODE for one direction (e.g., \cite[Th 6.1.11]{dojo-en}),
estimating an asymptoric expansion of $\uEE$ at a singular point (e.g., \cite{maddah-2017}),
finding a rational solution (e.g.,\cite{barkatou-2012}).
\section{Algorithms to derive an Hermite expansion for a holonomic activation function}  \label{sec:Hermite_expansion}


Han et al \cite[Th 2]{han-2022} gave a method to evaluate
the expectation $\uEE[\sigma(u) \sigma(v)]$ 
by utilizing the Hermite expansion of $\sigma(u)$.
Let $He_n(t)$ be probabilist's $n$-th Hermite polynomial
e.g., $He_0(t)=1$, $He_1(t)=t$, $He_2(t)=t^2-1$, $He_3(t)=t^3-3t$, $\ldots$.
The $n$-th coefficient 
of the Hermite expansion of $\sigma(u)$
is $\frac{c_n}{\sqrt{2\pi}n!}$ 
where
\begin{equation} \label{eq:hermite_coef_non_normalized}
  c_n =  \int_{-\infty}^\infty \sigma(u) He_n(u) \exp(-u^2/2) du.
\end{equation}
The function $\sigma(u)$ is expressed as
$ \sum_{n=0}^\infty \frac{c_n}{\sqrt{\pi} n! } He_n(x)$.
If $\sigma(u)$ is a holonomic function,
then $c_n$ satisfies a linear difference equation.
Creative telescoping algorithm (see, e.g., \cite{a=b-1996}, \cite{koutschan-2010})
or the integration algorithm of $D$-modules with the Mellin transformation
(see, e.g., \cite{ost-2003}) can be used to derive it.
By {\it difference HGM}, we mean that obtaining $c_n$ by the difference equation
of rank $r$
(recurrence relation) and initial values $c_0, \ldots, c_{r-1}$.

Examples of deriving Hermite expansions by these algorithms
are given in Section \ref{tksec:hermite-ReLU}
and Appendix \ref{sec:resin}.

The Hermite expansion expresses the dual activation as follows.
These results are by \cite{daniely-2016} and \cite{han-2022}.
\begin{theorem} {\rm \cite{daniely-2016}, \cite{han-2022}}
\begin{enumerate}
\item {\rm \cite{daniely-2016}}
If $\sigma$ is absolutely continuous and satisfies a homogenity
$\sigma(at)=|a|^q \sigma(t)$ for all $a, t \in \RR$,
the the dual activation is 
\begin{equation}
k_\sigma(c_1,c_2,r)= c_1^q c_2^q \sum_{j=0}^\infty 
  \left(\frac{c_j}{\sqrt{\pi} j! }\right)^2 r^j
\end{equation}
\item {\rm \cite{han-2022}}
Let $\sigma(t)$ be a polynomial $\sum_{j=0}^q a_j t^j$.
The dual activation is
\begin{equation}
k_\sigma(c_1,c_2,r)= \sum_{\ell=0}^q r_\ell(c_1)r_\ell(c_2) r^\ell
\end{equation}
where
\begin{equation}
 r_\ell(t)=\sum_{i=0}^{\lfloor (q-\ell)/2 \rfloor}
  \frac{a_{\ell+2i} (\ell+2i)!}{2^i i! \sqrt{\ell!}} t^{2i+\ell}.
\end{equation}
\end{enumerate}
\end{theorem}

As we will see later, the difference HGM is more efficient
than evaluating the integral (\ref{eq:hermite_coef_non_normalized})
indivisually.
Thus, the difference HGM strengthens the Hermite expansion method above.

\section{Faster Evaluation by the HGM}  \label{sec:fast-evaluation-by-the-HGM}

{\bf ``HGM all at once method''}.
When a definite integral with parameters satisfies a holonomic system,
it is possible to find integral values at many parameter points by a single run
of the Runge-Kutta method.
It is an advantage of utilizing the HGM.

Let us explain what it means by an example.
We use a Pfaffian system given in \cite[\S 6.2]{dojo-en}
for our example.
Put 
$P_1=\begin{pmatrix}
 0 & z_2/z_1 \\
 -z_1 z_2 & 1/z_1 \\
\end{pmatrix}
$
and
$P_2=\begin{pmatrix}
 0 & 1 \\
 -z_1^2 & 0 \\
\end{pmatrix}
$.
Consider the Pfaffian system
$$\pd{z_1}\bullet F=P_1 F,  \ 
  \pd{z_2}\bullet F=P_2 F
$$ 
where $F=(1,\pd{z_2})^T \bullet f$.
Note that $f=-\int_{\pi/2}^{z_1 z_2} \sin(t) dt =\cos(z_1 z_2)$ is a solution.
Suppose that we want to evalute $f$ at
$p_1=(\pi/2,1)$, $p_2=(\pi/2,2)$ and $p_3=(\pi,3)$.
Let $p_0=(z_1,z_2)=(\pi/2,0)$ be the starting point of the Runge-Kutta method.
The value of $F$ at the point is $(1,0)$.
We apply the Runge-Kutta method along the piecewise linear path
connecting $p_0, p_1, p_2, p_3$.
In other words, 
we solve the ODE
$$ 
 \frac{dF}{dt}=P_2(\pi/2,t)F  \quad \mbox{for $t \in [0,2]$}
$$
to find values of $F$ at $p_1$ and $p_2$,
and solve the ODE
$$ 
\frac{dF}{dt}=\left( P_1(z_1(t),z_2(t)) \frac{dz_1(t)}{dt}
                   + P_2(z_1(t),z_2(t)) \frac{dz_2(t)}{dt} \right) F
$$
along the path 
$z_1(t)=\pi/2+(\pi-\pi/2)(t-2)$,
$z_2(t)=2+(3-2)(t-2)$, $t \in [2,3]$
with the initial condition $F(p_2)$
to obtain the value $F$ at $p_3$.
It is faster than applying the Runge-Kutta method $3$ times
independently from $p_0$ to $p_i$, $i=1,2,3$.
See \cite{our-git} as to a sample code.

{\bf ``Taylor expansion with HGM'' method}.
Other approach to accelerate the HGM is to solve an ODE on a curve or a line
and determine the value of $f$ near a point of the curve or the line
by a Taylor expansion.
We explain this method by the example above.
We denote by $f_{ij}$ the derivative $\frac{\partial^{i+j} f}{\partial z_1^i \partial z_2^j}$.
Since $\pd{z_1} \bullet F=(f_{10},f_{11})$,
we can obtain the value of $(f_{10},f_{11})$ by evaluating $P_1 F$.
Note that $F=(f,f_{01})$.
Let $a$ be a point on the curve or the line.
Since the first order Taylor expansion at $z=a$ is
$f(a+h)=f(a)+f_{10}(a)h_1+f_{01}(a)h_2$,
we can express $f(a+h)$ in terms of the value of $F$ at $z=a$.
Values of higher order derivatives can also be expressed by $F$
by differentiating the Pfaffian system.
For example, we have 
$\pd{z_1}^2 \bullet F = \frac{\partial P_1}{\partial z_1} F + P_1 \frac{\partial F}{\partial z_1} = \frac{\partial P_1}{\partial z_1} F + P_1^2 F$
and
$\pd{z_1}^2 \bullet F = (f_{20},f_{21})$. 
Then, $f_{20}(a)$ and $f_{21}(a)$ can be expressed in terms of $F(a)$.


\section{HGM and HIE for ReLU and some other activator functions} \label{sec:relu}

\subsection{ReLU}

Let $\sigma(u)$ be ReLU (rectified linear unit) function;
$\sigma(u)={\rm max}(u,0)=uY(u)$
where $Y(u)$ is the Heaviside function.
Closed forms of 
the expectations (\ref{eq:expectation1}) and
(\ref{eq:expectation2}) in terms of 
$\arccos$ and $\sqrt{\quad}$ 
are known for the activator function ReLU.
Let 
$\Lambda$ be
$
\begin{pmatrix}
 c_1^2 & c_1 c_2 r \\
 c_1 c_2 r & c_2^2 \\
\end{pmatrix}
, c_1,c_2 \ge 0, |r| \le 1$
Then, 
\begin{equation}
 E_{(u,v)\sim N(0,\Lambda)}\lbrack \sigma(u)\sigma(v)\rbrack = \frac{r(\pi-\arccos(r))+\sqrt{1-r^2}}{2\pi}\cdot c_1 c_2
\end{equation}
\begin{equation}
 E_{(u,v)\sim N(0,\Lambda^(h))}\lbrack \dot\sigma(u) \dot\sigma(v)\rbrack = \frac{\pi - \arccos(r)}{2\pi}
\end{equation}
See, e.g., \cite{cho-2009}, \cite[Appendix I]{arora-2019} as to details.
By specializing our Theorem \ref{th:holonomic-umvn} to the case
$m=n=1$, we also obtain a closed form in a different form.
Although closed forms are already known,
we explain the algorithmic procedure of the HGM and HIE
by using the example of the ReLU, because a small example
is easier to understand the concept of our method.

\subsubsection{HGM for ReLU}

Since ReLU $\sigma(u)$ satisfies the linear differential equation
$ (u \pd{u}-1) \bullet \sigma(u)= uY'(u) + u Y(u) - \sigma(u)=0$
($u Y'(u) = u \delta(u)=0$)
as the distribution, we can apply \cite[Th 2]{koyama-takemura-2013} 
(Theorem \ref{koyama-takemura-th})
to obtain {\it a holonomic system} satisfied by $g(x)$.
See Appendix \ref{sec:proof-th:holonomic-umvn} as to details. 
The steps 1 and 2 of Algorithm \ref{alg:1} are performed by hand there.
See Appendix \ref{sec:relu-code1} to perform the steps 1 and 2 by a
computer algebra system.

Applying step 3 of Algorithm \ref{alg:1}, we obtain the following theorem
by Gr\"obner basis computations.
\begin{theorem}  \label{th:rank-and-pf-of-ReLU}
\begin{enumerate}
\item The holonomic system (\ref{eq:ann-gx-relu}) is of rank $2$.
Put 
$$F=(1,\pd{12})^T \bullet g.$$
Then we have the Pfaffian system
$\pd{x_{11}} \bullet F - P_{11}F=0,
 \pd{x_{12}} \bullet F - P_{12}F=0,
 \pd{x_{22}} \bullet F - P_{22}F=0
$.
Explicit form of the $2 \times 2$ matrix $P_{ij}$ is given in Appendix \ref{sec:relu-code2}.
\item 
The singular locus of the Pfaffian system (the denominator of the matrices $P_{ij}$)
is 
\begin{equation}
{x}_{11}     {x}_{22}  (   {x}_{12}^{ 2} -  {x}_{22}  {x}_{11}).
\end{equation}
\end{enumerate}
\end{theorem}
Note that the condition that $-X$ is positive definite
is 
\begin{equation}
x_{11} < 0, \quad x_{11}x_{22}-x_{12}^2 > 0.
\end{equation}

Let us proceed on the step 4 of Algorithm \ref{alg:1}.
We apply Proposition \ref{prop:series_at_ip}.
We take the special point $x=(x_{11},x_{12},x_{22})=(-1,0,-1)=:x_0$
where the integral splits to a product of
single integrals;
\begin{eqnarray}
&&\pd{x_{11}}^{d_{11}} \pd{x_{12}}^{d_{12}} \pd{x_{22}}^{d_{22}}
\bullet g(x) |_{x=x_0} \nonumber \\
&=&
2^{d_{12}} \int_0^\infty u^{1+2d_{11}+d_{12}} \exp(-u^2)du
\int_0^\infty v^{1+d_{12}+2d_{22}} \exp(-v^2)dv.
\end{eqnarray}
Since 
\begin{equation} \label{eq:moment}
\int_0^\infty u^m \exp(-u^2)du
=\frac{\Gamma\left(\frac{1+m}{2}\right)}{2},
\end{equation}
$\Gamma(n+1)=n!$,
and
$\Gamma\left(\frac{1}{2}+n\right)=
\frac{(2n-1)!!}{2^n} \sqrt{\pi}
$,
we have the series expansion of $g(x)$ at
$x=x_0$ as
\begin{eqnarray}
g(x)
&=&
\sum_{d_{12}: {\rm even}} \frac{2^{d_{12}}}{4 d!} 
     (d_{11}+d_{12}/2)! (d_{22}+d_{12}/2)! 
     {x'_{11}}^{d_{11}} x_{12}^{d_{12}} {x'_{22}}^{d_{22}} \nonumber \\
&+&\pi \sum_{d_{12}:{\rm odd}}
  \frac{2^{d_{12}}(2d_{11}+d_{12})!!(2d_{22}+d_{12})!!}
    { 2^{3+d_{11}+d_{22}+d_{12}} d!}
    {x'_{11}}^{d_{11}} x_{12}^{d_{12}} {x'_{22}}^{d_{22}} 
\end{eqnarray}
where $d! = {d_{11}! d_{12}! d_{22}!}$,
$x'_{ij}=x_{ij}+1$, and
$\sum_{d_{12}:{\rm even}}$ means
$\sum_{\{d \in \NN_0^3 \,|\, d_{12}:{\rm even}\}}$.
In particular,
we have
\begin{equation}
(1,\pd{x_{12}}) \bullet g(x) |_{x=x_0}
=\left(\frac{1}{4},\frac{\pi}{8} \right).
\end{equation}

Let us perform the step 5 of Algorithm \ref{alg:1}.
The domain $x_{11}<0, x_{11}x_{22}-x_{12}^2>0$
is convex. 
Choose a point $x_1$ in the domain.
Since the domain is convex,
we can restrict the Pfaffian system on the line
$rt+x_0$, $r=x_1-x_0$, $0 \leq t \leq 1$ and obtain
the ODE
\begin{eqnarray}
\frac{dF}{dt}&=&
\left( P_{11}(rt+x_0)r_{11}+P_{12}(rt+x_0)r_{12}+P_{22}(rt+x_0)r_{22} \right) F, \\
 F(0)&=&\left( \frac{1}{4},\frac{\pi}{8} \right)^T.
\end{eqnarray}
The fist element of $F(1)$ gives the value
of $g(x)$ at $x=x_1$.

An evaluation for ${\dot \sigma}$ by the HGM is explained
in Appendix \ref{sec:Heaviside-HGM}. 

\subsubsection{HIE for ReLU}

We restrict the system of differential equation to $x_{11}=x_{22}=-1$
and apply the HIE.
The ODE on $x_{11}=x_{22}=-1$ is
\begin{equation}  \label{eq:relu_restricted_to_11}
L \bullet f = 0, \quad 
L= (1-x_{12}^2) \pd{12}^2- 5 x_{12} \pd{12} - 4,
\end{equation}
which can be obtained by the restriction algorithm.
We want to solve it on $x_{12} \in [-1,1]$.
Let $e_0, e_1, \ldots, e_m$ be a basis and suppose that
$f = \sum_{\beta=0}^m f_\beta e_\beta$ where $f_\beta$'s are unknown coefficients. 
We suppose that the value of $f$ is given at $p_1$ and $p_2$ and the 
correponding values are $q_1$ and $q_2$.
Then, the loss function is
\begin{equation} \label{eq:loss-function-for-relu}
 \sum_{j=0}^n T_j \left( f_0 (L\bullet e_0)(t_j) + f_1 (L\bullet e_1)(t_j)
  + \cdots f_m (L\bullet e_m)(t_j) \right)^2
 + \mu \sum_{k=1}^2 \left(\sum_{\beta=0}^m f_\beta e_\beta(p_k) - q_k\right)^2
\end{equation}
Since $(L\bullet e_\beta)(t_j)$'s and $e_\beta(p_k)$'s are numbers,
it is a least square problem for an unknown vector $(f_0, f_1, \ldots, f_m)$.

\subsubsection{Hermite expansion of ReLU}  \label{tksec:hermite-ReLU}

Let $\sigma(u)=Y(u) u $ be the ReLU function where $Y(u)$ is the Heaviside 
function.
It satisfies $(u \pd{u}-1) \bullet \sigma(u)=0$ as a distribution.
Apply an algorithm to find annihilating ideal for a product of
distributions (see, e.g., \cite{ost-2003}),
we find an annihilating operator for
$ Y(u)u \cdot He_n(u) \exp(-u^2/2) $.
It also annihilated by the difference operator
$S_n^2-uS_n+(n+1)$ where $S_n$ is the shift operator for the variable $n$.
For example, we have $S_n \bullet He_n=He_{n+1}$.
Applying an algorithm to find linear difference equation for $c_n$
we obtain a difference operator annihilating $c_n$ as
$$
s_n^{2} +(  {n}- 1) 
$$
Initial values are 
$c_0=1$,
$c_1=\sqrt{\frac{\pi }{2}}$.


\subsection{Other activator distributions}
The HGM can be applied for any holonomic activator distributions.
To demonstrate this fact, 
we discuss on the HGM for an activator function 
\begin{equation} \label{eq:def_of_ReSin}
Y(u)\sin u ,
\end{equation}
which we will call ReSin,
in Appendix \ref{sec:resin}. 
Note that ReSin is not a smooth function.
We also discuss on GeLU in Appendix \ref{sec:GELU}, which is a smooth function.
Note that the Hermite expansion by \cite{han-2022}
provides very low approximate errors for smooth activator functions
like GeLU\footnote{Note that closed form of the expectations for GeLU is given
in \cite{tsuchida-2021} and \cite{han-2022}.},
but not for non-smooth distributions like ReLU and ReSin $Y(u)\sin u$.

\section{Experiments}

In this section, we perform experiments on our algorithms.
Before showing data of experiments, we explain what we mean
by learning and inference.
Let 
$\lbrace (x_{1}, y_{1}), (x_{2}, y_{2}), \ldots, (x_{N}, y_{N}) \rbrace$
be training data. 
Here $x_{i}$ is an input and $y_{i}$ is an output. 
Let $K(x,x^\prime)$ be the kernel function. 
The learning in the kernel method means that obtaining a matrix
$ H^\ast \in \RR^{N\times N} $ whose $i,j$ element is 
$K(x_i,x_j)$ from the training data.
We call $H^\ast$ the {\it kernel matrix}\/. 
We mean by inference obtaining an output from any input $x$
by the map 
\begin{equation}
 f(x) = (K(x,x_{1}), K(x,x_{2}), \ldots, K(x,x_{N}))(H^\ast + \lambda E)^{-1}(y_{1}, y_{2}, \ldots, y_{N})^T .
\end{equation}
Here $E$ is the identity matrix and the term $\lambda E$ is added 
to $H^\ast$ because there are cases where $H^\ast$ does not have the inverse.
We put $\lambda=0.01$ in our experiments.
The kernel function $K(x,x')$ is approximated by the function $\Theta$.

\begin{example}\rm  \label{ex:learn-sin}
Learning $\sin \pi x$ by a $2$ layer neural network with bias terms.
Input and output are $1$ dimensional.
Training data $(x_i,y_i)$ are equally spaced $15$ points $x_i$'s in $[-1,1]$ 
and their $y_i=\sin(\pi x_i)$'s values.
Inference uses equally spaced $20$ points in $[-1,1]$.
\end{example}

This small problem is our running example 
to make experiments by our methods for some activators.
\subsection{Comparison of our methods for some activators}

We compare the following methods for our running example \ref{ex:learn-sin}.
\begin{enumerate}
\item Closed forms of dual activation. It is referred as ``closed''.
\item Gauss-Hermite quadrature of \cite[3.3]{han-2022}. It is referred as ``GaussHerm'' or ``gh''.
We use an adaptive meshsize control with the relative error tolerance {\tt 1e-10}.
\item HGM of Algorithm \ref{alg:1}. It is referred as ``hgm''.
Our implementation uses {\tt scipy.solve\_ivp} function with {\tt rtol=1e-10}.
\item HGM all at once given in Section \ref{sec:fast-evaluation-by-the-HGM}.
     It is referred as ``all-at-once'' or ``aao''.
\end{enumerate}
Note that these methods except closed forms work 
for any holonomic activator functions. 
Monte-Carlo methods also work for any holonomic activator functions, 
but Monte-Carlo methods are relatively slow and inaccurate,
and then we do not make a comparison.

Timing data are taken on a machine with
Intel Core i5-12400 CPU (4.4 GHz).
We use numpy 1.24.4 and scipy 1.10.1 on wsl\footnote{We call this machine ``sw''.}.

\subsubsection{ReLU}
\begin{center}
 \begin{tabular}{l|c|c}
  Method & Training time (s)& Inference time (s)\\ \hline
  closed & 0.007292 & \\ \hline
  GaussHerm&1.500&  1.442\\ \hline
  hgm & 1.316& 1.953\\ \hline
  all-at-once & 4.352&  5.143\\ 
 \end{tabular} 
\end{center}
\begin{center}
\begin{minipage}{0.4\textwidth}
  \begin{tabular}[tb]{c|c}
  & Kernel error\\ \hline
  $\mbox{gh} - \mbox{hgm}$& 0.0010347\\ \hline
  $\mbox{gh} - \mbox{aao}$& 0.0010348\\ \hline
  $\mbox{hgm} - \mbox{aao}$& $2.7771 \times 10^{-8}$ \\ 
 \end{tabular}
\end{minipage}
\begin{minipage}{0.4\textwidth}
 \begin{tabular}{c|c}
    & Inference error \\ \hline
  GaussHerm& 0.97729 \\ \hline
  hgm& 0.96766\\ \hline
  all-at-once& 0.97460\\ 
 \end{tabular}
 \end{minipage}
\end{center}
Here, the kernel error is the Frobenius norm divided by the number of elements
of the difference of two kernel matrices.
The inference error is the mean square error between
the inference values and $\sin \pi x$ values at the $20$ points.
Note that we use the Ridge regression with $\lambda=0.01$ and then
the graph is a little different with that of $y=\sin \pi x$.
As is remarked in \cite[Th 2, 3.3]{han-2022},
the Gauss-Hermite quadrature or Hermite expansion provide
much lower approximation errors than non-smooth ones.
Since ReLU is not smooth, the inference by GaussHerm has a little larger 
inference error.
See Figure \ref{fig:fig_relu}.

\subsubsection{GeLU}
\begin{center}
 \begin{tabular}{l|c|c}
  Method & Training time (s) & Inference time (s) \\ \hline
  closed & \mbox{Not yet in our code} & \mbox{Not yet in our code} \\ \hline
  GaussHerm&41.25&   62.56\\ \hline
  hgm & 86.21& 120.2\\ \hline
  all-at-once & 32.89&  381.1\\
 \end{tabular} 
\end{center}
\begin{center}
\begin{minipage}{0.4\textwidth}
  \begin{tabular}[tb]{c|c}
  & Kernel error\\ \hline
  $\mbox{gh} - \mbox{hgm}$& $1.1718 \times 10^{-8}$\\ \hline
  $\mbox{gh} - \mbox{aao}$& $1.4983 \times 10^{-8}$\\ \hline
  $\mbox{hgm} - \mbox{aao}$& $1.5138 \times 10^{-8}$\\ 
 \end{tabular}
\end{minipage}
\begin{minipage}{0.4\textwidth}
 \begin{tabular}{c|c}
    & Inference error \\ \hline
  GaussHerm& 0.97163 \\ \hline
  hgm& 0.97166 \\ \hline
  all-at-once& 0.97163\\ \hline
 \end{tabular}
 \end{minipage}
\end{center}
The Gauss-Hermite quadrature provides low approximation errors and the inference
error of it is also small, because GeLU is a smooth activator.
The HGM and HGM all-at-once provide also low approximation errors as good as
GaussHerm.  See also Figure \ref{fig:fig_gelu}.
Note that our implementation of HGM has not yet included a code to evaluate
solutions near the singularity of an ODE, and then we use GaussHerm when
${\rm det}(\mbox{covariance matrix}) \leq 1 \times 10^{-3}$.
Note also that the closed form for GeLU is given in \cite{han-2022},
but the case of $\dot \sigma$ has not been implemented in our code.
It is the reason of ``not yet in our code'' in the table.

\subsubsection{ReSin}
\begin{center}
 \begin{tabular}{l|c|c}
  Method & Training time (s) & Inference time (s) \\  \hline
  closed & NA & NA \\ \hline
  GaussHerm& 3.916 &   4.949 \\ \hline
  hgm & 289.5  & 1005\\ \hline
  all-at-once & 21.07 &  23.39 \\ 
 \end{tabular} 
\end{center}

\begin{center}
\begin{minipage}{0.4\textwidth}
  \begin{tabular}[tb]{c|c}
  & Kernel error\\ \hline
  $\mbox{gh} - \mbox{hgm}$ & 0.0019103\\ \hline
  $\mbox{gh} - \mbox{aao}$& 0.0016427\\ \hline
  $\mbox{hgm} - \mbox{aao}$ & 0.00062839\\ 
 \end{tabular}
\end{minipage}
\begin{minipage}{0.4\textwidth}
 \begin{tabular}{c|c}
    & Inference error \\ \hline
  GaussHerm& 0.97328 \\ \hline
  hgm& 0.96874\\ \hline
  all-at-once&  0.95745\\ 
 \end{tabular}
 \end{minipage}
\end{center}
Since ReSin is not smooth, the Gauss-Hermite quadrature does not give
a low approximation error.
On the other hand, the HGM and HGM all-at-once give a good approximation
and HGM all-at-once works in a reasonable time.
Starting point of the HGM is $(-1,1/100,-1)$.
Note that the HGM is used only when 
$1 \times 10^{-3} \leq {\rm det}(\mbox{covariance matrix}) \leq 1$
and the Gauss-Hermite quadrature is used in other intervals
because of the same reason of the case of GeLU.
A closed form for ReSin is not known except an infinite series expression
in terms of contiguous family of Gauss hypergeometric functions
(Theorem \ref{th:holonomic-umvn}) as long as we know,
then the entry of closed form is ``NA'' (not available).

\begin{figure}[tbh]
\begin{minipage}[b]{0.49\columnwidth}
\centering
\includegraphics[width=5cm]{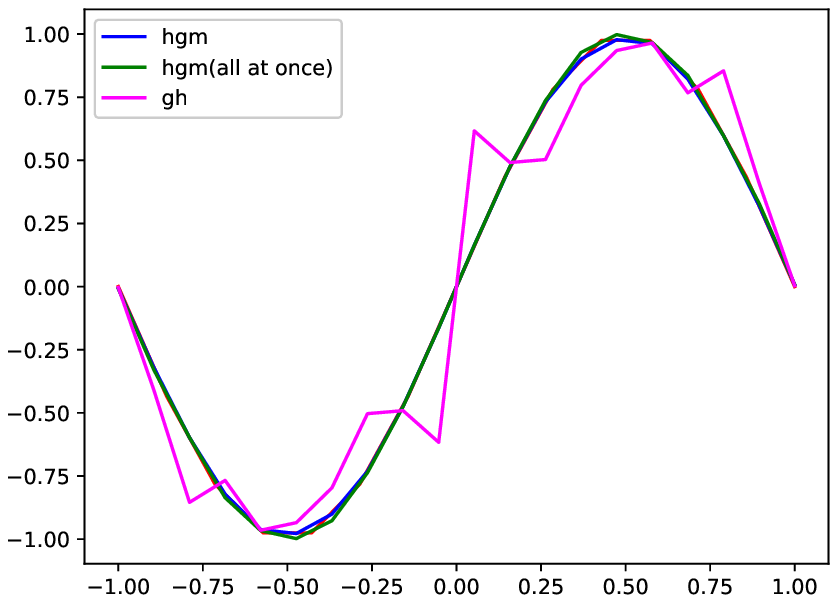}
\caption{Inference by ReLU} \label{fig:fig_relu}
\end{minipage}
\begin{minipage}[b]{0.49\columnwidth}
\centering
\includegraphics[width=5cm]{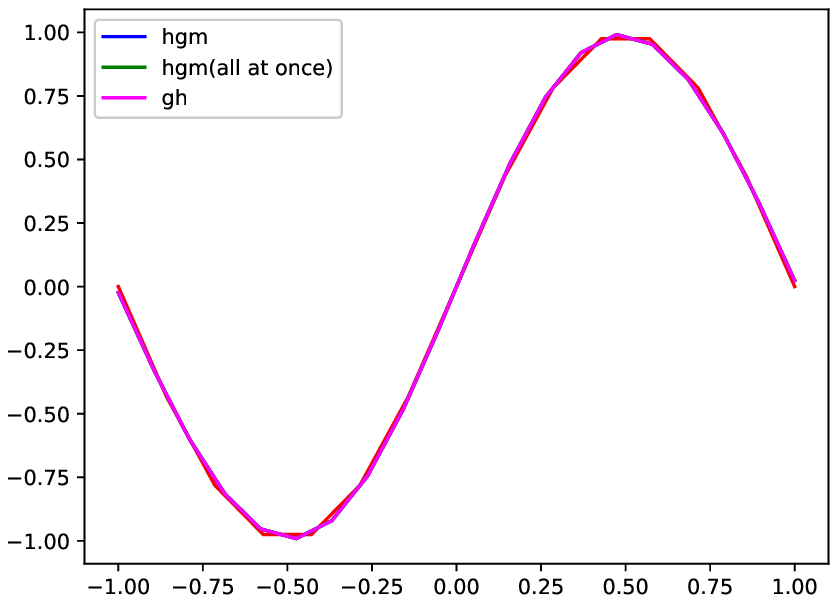}
\caption{Inference by GeLU} \label{fig:fig_gelu}
\end{minipage}
\begin{minipage}[b]{0.49\columnwidth}
\centering
\includegraphics[width=5cm]{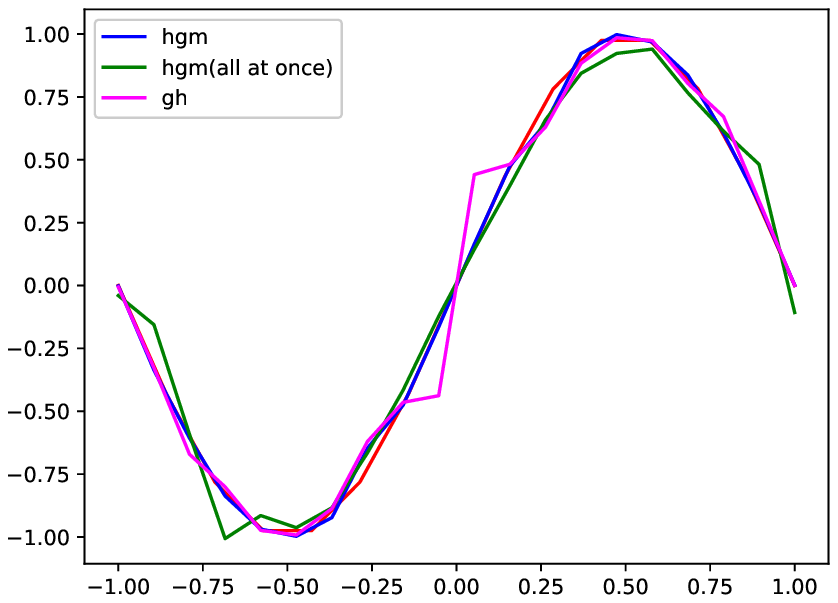}
\caption{Inference by ReSin} \label{fig:fig_resin}
\end{minipage}
\end{figure}

\subsection{Improvements of numerical solvers}

We evaluate the unnormalized expecation $\uEE$ (\ref{eq:unnormalized-expectation}) 
for the ReSin activator for $91$ points $(x,x')$, which were used
when calculating NTK for Example \ref{ex:learn-sin}.
We use the ``HGM all at once'' method 
in Section \ref{sec:fast-evaluation-by-the-HGM} for $91$ points.
Although there are more points $15^2-15=210$
to be evaluated, we remove points near the singular locus of the ODE
and groups clusters of nearby points into one point
to obtain the $91$ points.
Precisely, we take points satisfying $0.01 \leq {\rm det}(\mbox{covariance matrix}) \leq 1$. 
Two points whose distance is less than $1 \times 10^{-5}$ are represented by one point.

We will see that
the closer the path of integration of an ODE solver is to a straight line,
the faster the HGM will be.
Points are extracted for each ``step'' from among the sorted $91$ points.
We devide equally segments between these points and create new $91$ points.
When ``step'' is $1$, the set of new $91$ points is nothing but
the original one.
When ``step'' increases, the set of points has a distribution closer to a 
straight line. See Figure \ref{fig:fig-path-graph}.
The execution time of evaluating $\uEE$ at $91$ points becomes
faster when ``step'' increases. 
\begin{center}
\begin{tabular}{l|rrrrrr}
step     & 1&2&5&10&15&20 \\  \hline
time (s) & 2.065& 1.820& 1.550& 1.115& 0.771& 0.692 \\
\end{tabular}
\end{center}
See Figure \ref{fig:fig-path-timing}.
Note that evaluations at nearby points of a cluster 
can be done by the Taylor expansion with HGM method in Section \ref{sec:fast-evaluation-by-the-HGM} from the representative point of the cluster.

The timing data above is taken on AMD EPYC 7552 48-Core Processor of 1.5GHz
with 1T bytes memory without GPU
\footnote{We refer this machine as ``machine o3n''.}. 
We use the ODE solver {\tt gsl\_odeiv} rkf45 with {\tt rtol=1e-10} of the GNU scientific library 2.7.1 written in the language C
on the Debian GNU linux 12.2. 
The program is compiled by  gcc version is 12.2.0  with the option  -O3. 
Note that the execution time of the C code 
is about 1.9 times faster than a python code using the scipy {\tt solve\_ivp}
with {\tt rtol=1e-10} for the activator ReSin.
More precisely, the python code took 3.9166s with scipy version 1.10.1
and the C code took 2.065s.
\begin{figure}[htb]
\begin{minipage}[b]{0.49\columnwidth}
\begin{center}
\includegraphics[scale=0.4]{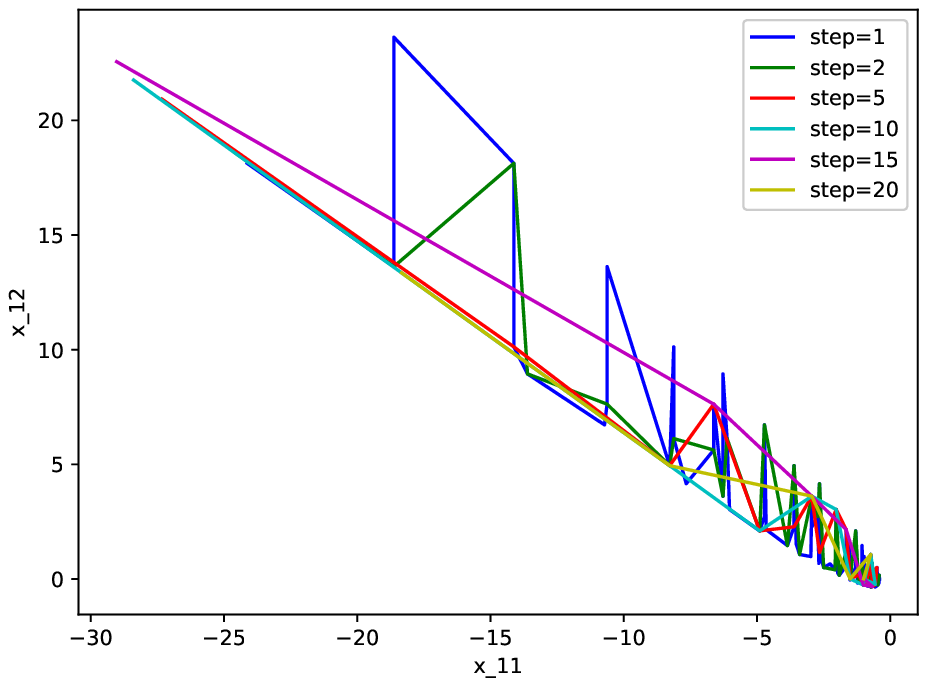}
\end{center}
\caption{Integration pathes of an ODE solver.} \label{fig:fig-path-graph}
\end{minipage}
\begin{minipage}[b]{0.49\columnwidth}
\begin{center}
\includegraphics[scale=0.4]{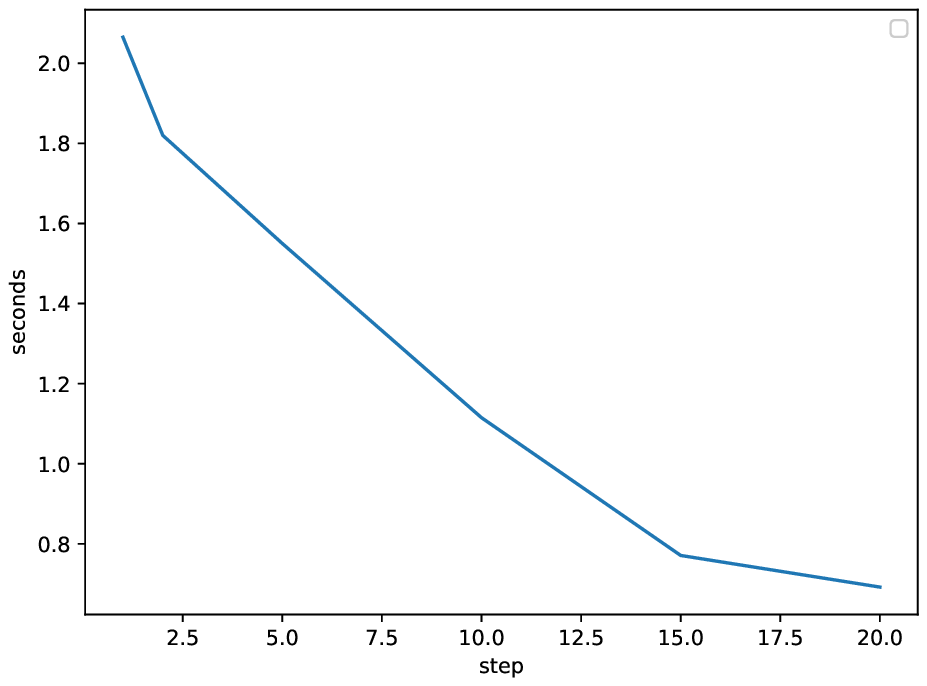}
\end{center}
\caption{Timing when the parameter ``step'' increases.} \label{fig:fig-path-timing}
\end{minipage}
\end{figure}

\bigbreak
\noindent
{\bf Acknowledgments} \\
The second author is supported in part by the JST CREST Grant Number JP19209317 and by JSPS KAKENHI Grant Number JP21K03270. 

\section{Appendix}

\subsection{Proof of Theorem \ref{th:holonomic-umvn}} \label{sec:proof-th:holonomic-umvn}

The function $t_1^m t_2^n$ is annihilated by
$ t_1 \pd{t_1} - m $,
and $t_2 \pd{t_2} -n $.
The distribution $t_1^m t_2^n Y(t_1)Y(t_2)$ is also annihilated by
these operators becuase $t_1 \pd{t_1} \bullet Y(t_1)= t_1 \delta(t_1)=0$.
Applying \cite[Th 1, 2]{koyama-takemura-2013}, we have the following
annihilating operators for the integral $\uEE[u^m v^n]$:
\begin{eqnarray}
&&\pd{1}(-y_1-2(x_{11}\pd{1}+x_{12}\pd{2}))-m, \label{eq:varphi1}\\
&&\pd{2}(-y_2-2(x_{12}\pd{1}+x_{22}\pd{2}))-n, \label{eq:varphi2}\\
&&\pd{12}-2\pd{1}\pd{2}, \label{eq:toric12} \\
&&\pd{11}-\pd{1}^2, \label{eq:toric11} \\
&&\pd{22}-\pd{2}^2  \label{eq:toric22},
\end{eqnarray}
where $\pd{ij}=\partial/\partial x_{ij}$,
$\pd{i}=\partial/\partial y_i$.
Let $I_1$ be the left ideal generated by the operators above.
We want to find elements
of $I_2=(I_1 + y_1 D + y_2 D) \cap \CC\langle x_{11},x_{12},x_{22},\pd{11},\pd{12},\pd{22}\rangle$
where $D=\CC\langle y_1,y_2,x_{11},x_{12},x_{22},\pd{1},\pd{2},\pd{11},\pd{12},\pd{22}\rangle$.
Expanding (\ref{eq:varphi1}), we have
\begin{eqnarray*}
&&-y_1 \pd{1}-1 -2(x_{11}\pd{1}^2+x_{12}\pd{1}\pd{2})-m \\
&\rightarrow&  -y_1 \pd{1} -2(x_{11}\pd{11}+(1/2)x_{12}\pd{12})-m-1, \quad
 \mbox{ by (\ref{eq:toric12}) and (\ref{eq:toric11})}. 
\end{eqnarray*}
Thus, $-2x_{11}\pd{11}-x_{12}\pd{12}-m-1$ is an element of $I_2$.
Analogusly, we can see that
$-x_{12}\pd{12}-2x_{22}\pd{22}-n-1$ is an element of $I_2$
from (\ref{eq:varphi2}).
Finally, we have
\begin{eqnarray*}
&& (\pd{12}-2\pd{1}\pd{2})^2 \\
&=&  4\pd{1}^2\pd{2}^2 -4\pd{1}\pd{2}\pd{12}+\pd{12}^2 \\
&\rightarrow&  4\pd{11}\pd{22}-2\pd{12}^2+\pd{12}^2 \quad
\mbox{ by (\ref{eq:toric12}) $\sim$ (\ref{eq:toric22})},  \\
&=& 4\pd{11}\pd{22}-\pd{12}^2
\end{eqnarray*}
We have proved 1.
Note: 
Changing variables $x_{11} \rightarrow 2 x_{11}$
and $x_{22} \rightarrow 2 x_{22}$,
we obtain the GKZ hypergeometric system of a standard form for the matrix
$A=\left(\begin{array}{ccc}
 2 & 1 & 0 \\
 0 & 1 & 2 \\
\end{array}\right)$.
We can also obtain the GKZ system for the unnormalized expectation $\uEE$
by a theory of integral representations of the GKZ system \cite{saiei-integral}.

Once the system of equations is expressed as a GKZ system,
we can apply a general procedure to obtain a series solution,
see, e.g., \cite{SST}.
A solution is written as
\begin{equation}
x_{11}^{\rho_{11}} x_{12}^{\rho_{12}} x_{13}^{\rho_{22}} f(z)
\end{equation}
where $z=\frac{x_{12}^2}{x_{11}x_{22}}$, 
$\rho_{11}=-(m+1)/2$, $\rho_{12}=0$, $\rho_{22}=-(n+1)/2$,
and
$f(z)$ is a solution of the Gauss hypergeometric differential
equation
$$ [\theta_z (\theta_z+1/2-1)-z (\theta_z+(m+1)/2) (\theta_z+(n+1)/2)] 
\bullet f = 0, \quad \theta_z = z \pd{z}.
$$
It has two independent solutions
${}_2F_1(\alpha,\beta,1/2;z)$
and $z^{1/2}{}_2F_1(\alpha+1/2,\beta+1/2,3/2;z)$.
Thus, we have proved the statement 2.

Note that the integral $\uEE[u^m v^n]$ and $\uEE[u^m v^n Y(u) Y(v)]$
are holomorphic at 
$x_{11}=x_{22}=-1, x_{12}=0$.
Restricting $x_{11}=x_{22}=-1$, we have a series expansion of 
$c_1 \varphi_1 + c_2 \varphi_2 = c_1 + c_2 x_{12} + O(x_{12}^2)$.
Note that we have
\begin{eqnarray*}
\uEE[u^m v^n](-1,0,-1)&=& \int_{-\infty}^\infty u^m \exp(-u^2) dv 
                         \int_{-\infty}^\infty v^n \exp(-v^2) dv \\
                     &=& \frac{(1+(-1)^{m})(1+(-1)^{n})}{4}\Gamma(\alpha)\Gamma(\beta), \\
\uEE[u^m v^n Y(u)Y(v)](-1,0,-1)&=&\int_{0}^\infty u^m \exp(-u^2) dv 
                         \int_{0}^\infty v^n \exp(-v^2) dv \\
                     &=& \frac{1}{4} \Gamma(\alpha)\Gamma(\beta)
\end{eqnarray*}
and
\begin{eqnarray*}
\frac{\partial \uEE[u^m v^n]}
     {\partial x_{12}}(-1,0,-1)
                     &=& 2 \int_{-\infty}^\infty u^{m+1} \exp(-u^2) dv 
                         \int_{-\infty}^\infty v^{n+1} \exp(-v^2) dv \\
                     &=& 2\frac{(1+(-1)^{m+1})(1+(-1)^{n+1})}{4}mn\Gamma(\alpha-1/2)\Gamma(\beta-1/2), \\
\frac{\partial \uEE[u^m v^n Y(u)Y(v)]}
     {\partial x_{12}}(-1,0,-1)
                     &=&2\int_{0}^\infty u^{m+1} \exp(-u^2) dv 
                         \int_{0}^\infty v^{n+1} \exp(-v^2) dv \\
                     &=& \frac{1}{2} \Gamma(\alpha+1/2)\Gamma(\beta+1/2).
\end{eqnarray*}
The constants $c_1, c_2$ are determined by these values and we obtain
the statements 3 and 4.

\subsection{Deriving a holonomic system by computer algebra} \label{sec:relu-code1}

We apply the restriction algorithm (see, e.g., \cite[\S 6.10]{dojo-en}
or \cite{oaku-1997})
and its implementation on
Risa/Asir \cite{risa-asir} 
to the left ideal $I$ generated by
\begin{eqnarray*}
&&\pd{y_1} (-y_1-2 x_{11}\pd{y_1}-2x_{12}\pd{y_2}) -1, \\
&&\pd{y_2} (-y_2-2 x_{12}\pd{y_1}-2x_{22}\pd{y_2}) -1, \\
&&\pd{12}-2\pd{y_1}\pd{y_2}, \\
&&\pd{11}-\pd{y_1}^2, \ \pd{22}-\pd{y_2}^2  \\
\end{eqnarray*}
in the ring of differential operators
$D=\QQ\langle x_{11},x_{12},x_{22}, y_1, y_2, \pd{11},\pd{12},\pd{22},\pd{y_1}, \pd{y_2} \rangle$.
The following ideal called the restriction ideal of $I$
to $y_1=y_2=0$:
\begin{equation} \label{eq:ann-gx-relu}
I':= (I+y_1D+y_2D)\cap \QQ\langle x_{11},x_{12},x_{22}, \pd{11},\pd{12},\pd{22}\rangle
\end{equation}
where $\pd{ij}=\pd{x_{ij}}$ 
These constructions are based on Gr\"obner bases computation 
in the Weyl algebra.
Here is a Risa/Asir code to obtain $I'$.
{\footnotesize
\begin{verbatim}
import("nk_restriction.rr");;
V=[y1,y2,x11,x12,x22]; DV=poly_dvar(V);
P1=poly_dmul(dy1,-y1-2*x11*dy1-2*x12*dy2,V)-1;
P2=poly_dmul(dy2,-y2-2*x12*dy1-2*x22*dy2,V)-1;
I=[P1,P2,dx11-dy1^2,dx22-dy2^2,dx12-2*dy1*dy2];
dp_gr_print(1);
Iprime=nk_restriction.restriction_ideal(I,V,DV,[1,1,0,0,0]);
\end{verbatim}
}

\subsection{Proof of Theorem \ref{th:rank-and-pf-of-ReLU}} \label{sec:relu-code2}

We translate $I'$ to a Pfaffian system 
by a Gr\"obner basis computation in the ring
of differential operators with rational function coefficients
(the rational Weyl algebra).
See, e.g., \cite[\S 6.2]{dojo-en} on the translation.
Here is a Risa/Asir code to translate $I'$ to a Pfaffian system.
{\footnotesize
\begin{verbatim}
import("yang.rr");;
VV=[x11,x12,x22]; DVV=poly_dvar(VV);
yang.define_ring(["partial",VV]);
RII=map(dp_ptod,Iprime,DVV);
yang.verbose();
RG=yang.buchberger(RII);;
Std=[1,dx12];
Pf=yang.pfaffian(map(dp_ptod,Std,DVV),RG);
\end{verbatim}
}

\begin{eqnarray*}
P_{11}&=&
\left(\begin{array}{cc}
 \frac{ - 1} {  {x}_{11}}&  \frac{   - \frac{ 1} { 2}   {x}_{12}} {  {x}_{11}} \\
 \frac{   2   {x}_{12}} {   {x}_{11}  (   {x}_{12}^{ 2} -  {x}_{22}  {x}_{11})}&  \frac{   \frac{ 1} { 2}  (   2   {x}_{12}^{ 2} +   3  {x}_{22}  {x}_{11})} {   {x}_{11}  (   {x}_{12}^{ 2} -  {x}_{22}  {x}_{11})} \\
\end{array}\right), \\
P_{12}&=& \left(\begin{array}{cc}
0&  1 \\
 \frac{-4} {    {x}_{12}^{ 2} -  {x}_{22}  {x}_{11}}&  \frac{   - 5   {x}_{12}} { (   {x}_{12}^{ 2} -  {x}_{22}  {x}_{11})} \\
\end{array}\right), \\
P_{22}&=& \left(\begin{array}{cc}
 \frac{ - 1} { {x}_{22}}&  \frac{   - \frac{ 1} { 2}   {x}_{12}} {  {x}_{22}} \\
 \frac{   2   {x}_{12}} {   {x}_{22}  (   {x}_{12}^{ 2} -  {x}_{22}  {x}_{11})}&  \frac{   \frac{ 1} { 2}  (   2   {x}_{12}^{ 2} +   3  {x}_{22}  {x}_{11})} {   {x}_{22}  (   {x}_{12}^{ 2} -  {x}_{22}  {x}_{11})} \\
\end{array}\right).
\end{eqnarray*}

\subsection{HGM for the Derivative of ReLU (Heaviside Function)} \label{sec:Heaviside-HGM}

Let $\sigma(u)$ be ReLU (rectified linear unit) function;
$\sigma(u)={\rm max}(u,0)$.
The derivative of $\sigma(u)$ is the Heaviside function
$Y(u)$.
Put 
$$
 g(x)=\int_{\RR^2} Y(u) Y(v) \exp(x_{11} u^2+2x_{12}uv+x_{22}v^2) du dv.
$$
The expectation for the Heaviside function $g(x)/Z(x)$ is called 
{\it the orthant probability}.
Koyama and Takemura \cite{koyama-takemura-2013} gave a method
to evaluate it in general dimensions by the HGM.
In particular, they show that the $2$-dimensional orthant probability satisfies a Pfaffian system of rank $4$.
Since the average of the normal distribution we consider
is $0$,
the orthant probability satisfies a simpler Pfaffian system
than the system given by them.
It follows from \cite[Th 1, Th 2]{koyama-takemura-2013} that
$g(x)$ is annihilated by the left ideal 
\begin{equation} \label{eq:ann-gx-relu-diff}
I'=I \cap \RR \langle x_{11},x_{12}, x_{22}, \pd{11}, \pd{12}, \pd{22} \rangle 
\end{equation}
where $I$ is generated by
\begin{eqnarray*}
&&\pd{y_1} (-y_1-2 x_{11}\pd{y_1}-2x_{12}\pd{y_2}), \\
&&\pd{y_2} (-y_2-2 x_{12}\pd{y_1}-2x_{22}\pd{y_2}),\\
&&\pd{12}-2\pd{y_1}\pd{y_2}, \\
&&\pd{11}-\pd{y_1}^2, \ \pd{22}-\pd{y_2}^2.  \\
\end{eqnarray*}

\begin{theorem}  \label{th:rank-and-pf-of-ReLU-diff}
\begin{enumerate}
\item The holonomic system (\ref{eq:ann-gx-relu-diff}) is of rank $2$.
\item Put 
$$F=(1,\pd{12})^T \bullet g.$$
Then we have the Pfaffian system
$\pd{x_{11}} \bullet F - P_{11}F=0,
 \pd{x_{12}} \bullet F - P_{12}F=0,
 \pd{x_{22}} \bullet F - P_{22}F=0
$
where
\begin{eqnarray*}
d_1&=&  -  {x}_{12}^{ 2} +  {x}_{22}  {x}_{11}, \\
P_{11}&=&
\left(\begin{array}{cc}
\frac{- 1/ 2}{ {x}_{11}}& \frac{ -  1/2  {x}_{12}}{ {x}_{11}} \\
\frac{ -  1/2  {x}_{12}}{  {d}_{1}  {x}_{11}}& \frac{  -  1/2   {x}_{12}^{ 2} -  {x}_{22}  {x}_{11}}{  {d}_{1}  {x}_{11}} \\
\end{array}\right), \\ 
P_{12}&=&
\left(\begin{array}{cc}
0&  1 \\
\frac{ 1}{ {d}_{1}}& \frac{  3  {x}_{12}}{ {d}_{1}} \\
\end{array}\right), \\ 
P_{22}&=&
\left(\begin{array}{cc}
\frac{- 1/ 2}{ {x}_{22}}& \frac{ -  1/2  {x}_{12}}{ {x}_{22}} \\
\frac{ -  1/2  {x}_{12}}{  {d}_{1}  {x}_{22}}& \frac{  -  1/2   {x}_{12}^{ 2} -  {x}_{22}  {x}_{11}}{  {d}_{1}  {x}_{22}} \\
\end{array}\right).
\end{eqnarray*}
\end{enumerate}
\end{theorem}

By (\ref{eq:moment}) and an analogous discussion with ReLU case,
We have
$
(1,\pd{x_{12}} \bullet g(x) |_{x=x_0}
=\left(\frac{\pi}{4},\frac{1}{2} \right)
$,
which gives an accurate initial value to solve ODE's for the HGM.

\subsection{HGM and difference HGM for ReSin} \label{sec:resin}

We call the function 
$\sigma(y)=Y(u) \sin(u)$ the ReSin ({\it rectified sine}) function
where $Y(x)$ is the Heaviside function.
Note that $\sigma(u)$ is not a differentiable function and
is a tempered distribution.

\subsubsection{HGM for ReSin}
Since ReSin $\sigma(u)$ satisfies the linear differential equation
$ u^2 (\pd{u}^2 + 1) \bullet \sigma(u)=0$
as the distribution, we can apply \cite[Th 2]{koyama-takemura-2013} 
(Theorem \ref{koyama-takemura-th})
to obtain {\it a holonomic system} satisfied by $g(x)$.

Applying steps 2 and 3 of Algorithm \ref{alg:1}, we obtain the following theorem
by Gr\"obner basis computations.
It took 13.114s on the machine o3n.
\begin{theorem}  \label{th:rank-and-pf-of-ReSin}
\begin{enumerate}
\item The holonomic system for ReSin of the form (\ref{eq:ann-gx-relu}) 
is of rank $8$
and standard monomials can be taken as
\begin{equation}  \label{eq:std_of_ReSin}
(1,\pd{11},\pd{12},\pd{22},\pd{11}^2,\pd{12}^2,\pd{22}^2,\pd{11}\pd{12}\pd{22}).
\end{equation}
Put 
$$F=(1,\pd{11},\pd{12},\pd{22},\pd{11}^2,\pd{12}^2,\pd{22}^2,\pd{11}\pd{12}\pd{22})^T \bullet g.$$
Then we have the Pfaffian system
$\pd{x_{11}} \bullet F - P_{11}F=0,
 \pd{x_{12}} \bullet F - P_{12}F=0,
 \pd{x_{22}} \bullet F - P_{22}F=0
$.
Explicit form of the $8 \times 8$ matrix $P_{ij}$ is given in \cite{our-git}.
\item 
The singular locus of the Pfaffian system (the denominator of the matrices $P_{ij}$)
is 
\begin{equation}
{x}_{11}^{4} x_{12}    {x}_{22}^{4}    (   {x}_{12}^{ 2} -  {x}_{22}  {x}_{11})^{4}   (   {x}_{12}^{ 2} +  {x}_{22}  {x}_{11}).
\end{equation}
\end{enumerate}
\end{theorem}

\subsubsection{Deriving Hermite expansion by Difference HGM}

Let $\sigma(u)=Y(u) \sin(u) $ be the ReSin function.
It satisfies $u^2 (\pd{u}^2 + 1) \bullet \sigma(u)=0$ as a distribution.
Applying an algorithm to find annihilating ideal for a product of
distributions (see, e.g., \cite{ost-2003}),
we find an annihilating operator for
$ Y(u) \sin(u) \cdot He_n(u) \exp(-u^2/2) $.
It is also annihilated by the difference operator
$S_n^2-2uS_n+2(n+1)$ where $S_n$ is the shift operator for the variable $n$.
For example, we have $S_n \bullet He_n=He_{n+1}$.
Applying an algorithm to find linear difference equation for $c_n$
we obtain a difference operator annihilating $c_n$ as
$$
   s_n^{6}  + 2  (  {n}+ 3)) s_n^{4}  + (    {n}^{ 2} +  5  {n}+ 7)) s_n^{2} 
 +  (  {n}+ 1)  (  {n}+ 2) .
$$
Initial values for the difference equation are
$
c_0=\sqrt{2} F\left(\frac{1}{\sqrt{2}}\right),
c_1=\sqrt{\frac{\pi }{2 e}},
c_2=1-\sqrt{2} F\left(\frac{1}{\sqrt{2}}\right),
c_3=-\sqrt{\frac{\pi }{2 e}},
c_4=\sqrt{2} F\left(\frac{1}{\sqrt{2}}\right)-2,
c_5=\sqrt{\frac{\pi }{2 e}}
$
where $F$ is the Dawson's function $F$\footnote{\url{https://en.wikipedia.org/wiki/Dawson_function}} 
and
$e$ is Euler's number (Napier's number).
These initial values
are expressed in terms of special values of Dawson's integral.
Here is a session by Mathematica.
{\footnotesize
\begin{verbatim}
Mathematica 11.2.0 Kernel for Linux x86 (64-bit)
Copyright 1988-2017 Wolfram Research, Inc.

In[1]:= hermiteE[n_,x_]:=2^(-n/2)*HermiteH[n,x/2^(1/2)]
In[2]:= Integrate[Sin[u]*hermiteE[0,u]*Exp[-u^2/2],{u,0,Infinity}]
                           1
Out[2]= Sqrt[2] DawsonF[-------]
                        Sqrt[2]

In[3]:= Integrate[Sin[u]*hermiteE[1,u]*Exp[-u^2/2],{u,0,Infinity}]  
             Pi
Out[3]= Sqrt[---]
             2 E
\end{verbatim}
}

To avoid errors in the numerical calculation of $c_n$ by the recurrence formula,
we put $d_0=\sqrt{2} F\left(\frac{1}{\sqrt{2}} \right)$
and
$d_2=\frac{\sqrt{\pi}}{\sqrt{2e}}$ and solve the recurrence
by the rational arithmetic in $\QQ[d_1,d_2]$ and finally replace 
$d_1, d_2$ by their approximate numerical values.
There are several methods to avoid such errors of difference HGM
(see, e.g., \cite{tgkt-2020}).

Our difference HGM gives all $c_k$, $k \leq 100$ in 0.009315s.
On the other hand, it takes 1.6998s by Mathematica
just to obtain only $c_{99}$ on o3n.


\subsection{HGM for GeLU} \label{sec:GELU}
Let $\sigma(u)$ be the Gaussian error linear unit (GeLU) \cite{gelu-2016}
\begin{equation}
x (1+{\rm erf}(x))
\end{equation}
where the error function ${\rm erf}(x)$ is
$\frac{2}{\sqrt{\pi}} \int_0^x \exp(-t^2)dt$.
Note that the GeLU of \cite{gelu-2016} is
$\frac{1}{2} x (1+{\rm erf}(x/\sqrt{2})$ which agrees with ours
by changing the variable $x/\sqrt{2}$ to $x$ and multiplying a scalar.

\subsubsection{Expectation for GeLU}
In this section, we will evaluate the expectations 
by the HGM. 
In other words, we will numerically evaluate the integral
(\ref{eq:expectation-sigma2})
by the HGM.
Since we have explained how to apply the framework of the HGM 
to the evaluation for ReLU in Section \ref{sec:relu},
we only explain only the differences.

\begin{proposition} \label{prop:gelu-ode}
\item The GeLU $\sigma(u)$ is annihilated by the linear ordinary differential operator
\begin{equation}  \label{eq:gelu-ode1}
u^2 \pd{u}^2 -2 u (1-u^2) \pd{u}+2 (1-u^2)
\end{equation}  \label{eq:gelu-ode}
\end{proposition}
Since GeLU $\sigma(u)$ is a holonomic function by the proposition,
we can apply \cite[Th 2]{koyama-takemura-2013} 
(Theorem \ref{koyama-takemura-th}) to obtain a holonomic system satisfied by $g(x)$ (\ref{eq:expectation-sigma2}).

\begin{theorem}  \label{th:gelu-pf}
Let $Z(x)=\pi/\sqrt{x_{11}x_{22}-x_{12}^2}$ be the normalizing constant for the normal distribution
of the covariance matrix  $(-2 x)^{-1}$,
$x = \begin{pmatrix} x_{11} & x_{12} \\ x_{12} & x_{22} \\ \end{pmatrix}$
and of the average $0$.
\begin{enumerate}
\item The holonomic systems satisfied by the expectation multiplied by the normalizing constant 
$g_1(x)=\uEE[\sigma(u)\sigma(v)]=Z(x) E[\sigma(u)\sigma(v)]$ ( $g(x)$ of (\ref{eq:expectation-sigma2}) for the case that $\sigma$ is GeLU)
and 
$g_2(x)=\uEE[{\dot\sigma}(u){\dot\sigma}(v)]Z(x) E[{\dot \sigma}(u){\dot \sigma}(v)]$
( $g(x)$ of (\ref{eq:expectation-sigma2}) for the case that $\sigma$ is the derivative of GeLU)
are of rank $8$.
\item The singular locus of the Pfaffian system
for $\uEE[\sigma(u)\sigma(v)]$ with respect to
$$ S=(1,\pd{11},\pd{12},\pd{22}, \pd{12}\pd{22},\pd{11}\pd{12},\pd{11}\pd{22},
  \pd{11}\pd{12}\pd{22})^T \bullet g_i
$$
is the union of zeros of  
$
  d_1={x}_{22},
  d_2={x}_{22}- 1,
  d_3={x}_{11},
  d_4={x}_{11}- 1,
   d_5={x}_{12}^{ 2} -  {x}_{22}  {x}_{11},
   d_6={x}_{12}^{ 2}   -  {x}_{22}  {x}_{11}+ {x}_{22},
   d_7={x}_{12}^{ 2} +  (  - {x}_{22}+ 1)  {x}_{11},
   d_8={x}_{12}^{ 2} +   (  - {x}_{22}+ 1)  {x}_{11}+  {x}_{22}- 1,
   d_9={x}_{12}^{ 4} +   (  -  {x}_{22}^{ 2} + {x}_{22})   {x}_{11}^{ 2} +  (   {x}_{22}^{ 2} - {x}_{22})  {x}_{11}
$.
\end{enumerate}
\end{theorem}
This theorem can be proven in an analogously way as the proof of Theorem 
\ref{th:rank-and-pf-of-ReLU} as follows.
\begin{enumerate}
\item Derive the holonomic system 
by applying \cite{koyama-takemura-2013} (Theorem \ref{koyama-takemura-th})
and Proposition \ref{prop:gelu-ode}.
\item Translate the holonomic system by applying the restriction algorithm
and an algorithm to obtain a Pfaffian system from the holonomic system.
Risa/Asir codes to perform them are at \cite{our-git}.
\end{enumerate}
These are performed in 338.8s on o3n.

\subsubsection{Proof of Proposition \ref{prop:gelu-ode}} \label{sec:proof-prop-gelu-ode}

The Erf function ${\rm erf}(u)$ satisfies
$\pd{u} \bullet {\rm erf}(u) = \frac{2}{\sqrt{\pi}} \exp(-u^2)$.
Then, it is annihilated by the operator
$ (\pd{u} + 2u) \pd{u} $,
which also annihilates 
$f_1(u)=1+{\rm erf}(u)$.
Let us derive the ODE satisfied by $\sigma(u)=f_1(u)f_2(u)$.
The function $f_2(u)=u$ is annihilated by $u \pd{u}-1$.
We derive a linear dependent relation
for $\sigma$,  $\sigma'=f_1' f_2 + f_1 f_2'$ and
$\sigma''=f_1''f_2 + 2 f_1' f_2' + f_1 f_2''$.
Since $f_1$ satisfies the rank $2$ ODE and $f_2$ satisfies the rank $1$
ODE,
we can express $\sigma'$ and $\sigma''$
in terms of $f_1 f_2$, $f_1' f_2$
by replacing $f_1''$  by $-2u f_1'$,
$f_2'$ by $f_2/u$ and $f_2''$ by $0$.
In fact, we have
$\sigma=f_1 f_2$,
$\sigma'=f_1' f_2 + f_1 f_2/u$,
$\sigma''=-2u f_1' f_2 + 2 f_1' f_2/u$
and these $3$ functions are linearly dependent over the rational function
field $\CC(u)$.
Put the coefficients of the dependency as $c_i(u)$.
Then, we have
$$ c_2 \sigma'' + c_1 \sigma' + c_0 \sigma
 = ( c_1/u + c_0) f_1 f_2 + ((-2u+ 2/u) c_2 + c_1) f_1' f_2 = 0.
$$
Assuming $f_1 f_2$ and $f'_1 f_2$ are linearly independent,
we have $c_1/u+c_0=0$ and $(-2u+2/u)c_2+c_1=0$.
Put $c_1=1$.
Then, $c_0=-1/u$ and $c_2=\frac{1}{2u-2/u}=\frac{u}{2u^2-2}$. 
Thus, we obtain (\ref{eq:gelu-ode1}) by multiplying 
$2 u (u^2-1)$.

\subsubsection{Evaluation of the expecation of the derivative of GeLU}

We retain the notation of \ref{sec:proof-prop-gelu-ode}.
The unnormalized expecation (\ref{eq:unnormalized-expectation})
$\uEE[(f(u)+g(u))(f(v)+g(v))]$
is a sum of 
$\uEE[f(u)f(v)]$,
$\uEE[f(u)g(v)]$,
$\uEE[g(u)f(v)]$, and
$\uEE[g(u)g(v)]$
where
$f(u)=u \, {\rm erf}'(u)$ and $g(u)=1+{\rm erf}(u)$.
Since these functions satisfy simpler ODE's, evaluation becomes faster than 
utilizing the holonomic system for $g_2(x)$ in Theorem \ref{th:gelu-pf}.

\begin{theorem} \label{th:gelu-diff-by-sum}
\begin{enumerate}
\item The unnormalized expectation $\uEE[f(u)f(v)]$
is equal to 
\begin{equation}
 \frac{4 x_{12}}{2((x_{11}-1)(x_{22}-1)-x_{12}^2)^{3/2}}.
\end{equation}
\item The holonomic systems satisfied by $\uEE[f(u)g(v)]$ and
$\uEE[g(u)g(v)]$ are of rank $2$.
The singular locus of the Pfaffian system with respect to $(1,\pd{12})^T \bullet \uEE[f(u)g(v)]$
is the union of zeros of
${x}_{22}$, 
${x}_{22}- 1$,
${x}_{11}- 1$,
$d_1={x}_{12}^{ 2}   -  {x}_{22}  {x}_{11}+ {x}_{22}$,
$d_2={x}_{12}^{ 2} +   (  - {x}_{22}+ 1)  {x}_{11}+  {x}_{22}- 1$.
The singular locus of the Pfaffian system with respect to 
$(1,\pd{12})^T \bullet \uEE[g(u)g(v)]$
is the union of zeros of
$d_1$,
$d_2$,
$d_3= {x}_{12}^{ 2} -  {x}_{22}  {x}_{11}$,
$d_4={x}_{12}^{ 2} +  (  - {x}_{22}+ 1)  {x}_{11}$,
$d_5={x}_{12}^{ 4} +   (  -  {x}_{22}^{ 2} + {x}_{22})   {x}_{11}^{ 2} +  (   {x}_{22}^{ 2} - {x}_{22})  {x}_{11}$.
These Pfaffian systems are given at \cite{our-git}.
\item Values of $\uEE[f(u)g(v)]$ and
$\uEE[g(u)g(v)]$ at $x=\begin{pmatrix} -1 & 0 \\ 0 & -1 \\ \end{pmatrix}$
are $0$ and $\frac{1}{2}$ respectively.
Values of $\pd{12} \bullet \uEE[f(u)g(v)]$ and
$\pd{12}\bullet \uEE[g(u)g(v)]$ at $x=\begin{pmatrix} -1 & 0 \\ 0 & -1 \\ \end{pmatrix}$
are $\pi$ and $1$ respectively.
\end{enumerate}
\end{theorem}

\begin{proof}
1. $\uEE[f(u)f(v)]$ is a moment of the normal distribution and is easy to 
obtain an explicit form. \\
2. The Pfaffian systems are obtained in an analogous way as the proof of
Theorem \ref{th:rank-and-pf-of-ReLU}. \\
3. The integrals are products of single integrals for the special value
of $x$ in the statement.
We can obtain an explicit form of these single integrals, e.g.,
with a help of Mathematica.
\end{proof}
Derivation of Pfaffian systems is done in 10.066s on o3n.
\subsection{Degenerated Normal Distribution}

When ${\rm det}(\Sigma)=0$,
the HGM for the double integral (\ref{eq:expectation-sigma2}) 
cannot be applied because ${\rm det}(X)=1/{\rm det}(-\Sigma/2)$ 
becomes infinity.
Following the discussion of \cite[p.30]{anderson},
we will derive a single integral representation for the expectation
$E[\sigma(u)\sigma(v)]$.

Let $\Sigma$ be the covariance matrix of rank $1$.
Then there exists a non-sigular symmetric matrix $B$ such that
$B \Sigma B = \begin{pmatrix} 1 & 0 \\ 0 & 0 \end{pmatrix}$.
Let $(c_1,c_2)^T$ be the first column vector of the matrix 
$B^{-1}$.
Then, the expectation for the activator $\sigma$ is expressed as
\begin{equation} \label{eq:expectation-sigma2-degenerated}
\frac{1}{\sqrt{2\pi}}\int_{-\infty}^\infty
 \sigma(c_1 z) \sigma(c_2 z) \exp\left(-\frac{z^2}{2} \right) dz
\end{equation}
by \cite[p.30]{anderson}.
To obtain the expectation for ${\dot \sigma}$,
we may replace $\sigma$ by ${\dot \sigma}$ in (\ref{eq:expectation-sigma2-degenerated}).
In order to evaluate the integral, we may utilize the HGM for $c_1$ and $c_2$
or an efficient numerical integrator for single integrals.

\end{document}